\documentclass{article}

\PassOptionsToPackage{numbers,sort&compress}{natbib}

\usepackage[final]{neurips_2022}

\usepackage[utf8]{inputenc} 
\usepackage[T1]{fontenc}    
\usepackage{microtype}      

\usepackage[table]{xcolor}
\colorlet{citeblue}{blue!50!black}
\definecolor{mplblue}{HTML}{1f77b4}
\definecolor{mplorange}{HTML}{ff7f0e}
\definecolor{mplgreen}{HTML}{2ca02c}
\definecolor{mplred}{HTML}{d62728}
\definecolor{mplpurple}{HTML}{9467bd}

\usepackage{enumitem}
\setitemize{noitemsep,topsep=0pt,parsep=0pt,partopsep=0pt} 

\usepackage{amsmath}
\usepackage{amssymb}
\usepackage{amsfonts}       
\usepackage{nicefrac}       
\usepackage{mathtools}
\usepackage{relsize}
\usepackage{bm}

\usepackage{etoolbox}
\usepackage{pgfplots}
\usepgfplotslibrary{groupplots}
\usepackage{tikz}
\usetikzlibrary{tikzmark}
\usetikzlibrary{calc}
\newcommand{\colordot}[2][0.75ex]{\tikz[baseline=-0.66ex]\draw[#2,fill=#2,radius=#1,fill opacity=0.4] (0,0) circle ;}%
\newcommand{\colorcircle}[2][0.75ex]{\tikz[baseline=-0.66ex]\draw[#2,fill=white,radius=#1,very thick,fill opacity=1.0] (0,0) circle ;}
\newcommand{\filledcolordot}[2][0.75ex]{\tikz[baseline=-0.66ex]\draw[#2,fill=#2,radius=#1] (0,0) circle ;}%

\usetikzlibrary{external}

\usepackage{booktabs}
\usepackage{siunitx}
\usepackage{multirow}

\usepackage{graphicx}
\usepackage{pgfplots}
\pgfplotsset{compat=1.15}
\usepackage{caption}
\usepackage[labelformat=simple,subrefformat=simple]{subcaption}
\usepackage{adjustbox}

\usepackage{algorithm}
\usepackage[noend]{algpseudocode}
\algrenewcommand{\algorithmiccomment}[1]{\hfill {\small \textcolor{darkgray}{$\mathsmaller \vartriangleright$ #1}}}  
\algrenewcommand\algorithmicindent{1.5em}   
\algrenewcommand\alglinenumber[1]{\small {\textcolor{darkgray}{#1}}} 

\makeatletter
\expandafter\patchcmd\csname\string\algorithmic\endcsname{\itemsep\z@}{\itemsep=0.25ex}{}{}
\newcommand\fs@booktabsruled{%
  \def\@fs@cfont{\bfseries\strut}\let\@fs@capt\floatc@ruled
  \def\@fs@pre{\hrule height\heavyrulewidth depth0pt \kern\belowrulesep}%
  \def\@fs@mid{\kern\aboverulesep\hrule height\lightrulewidth\kern\belowrulesep}%
  \def\@fs@post{\kern\aboverulesep\hrule height\heavyrulewidth\relax}%
  \let\@fs@iftopcapt\iftrue
}
\makeatother
\floatstyle{booktabsruled}
\restylefloat{algorithm}
\usepackage{amsthm}
\usepackage{thmtools}
\usepackage{thm-restate}
\newtheoremstyle{theorem-style}
{\topsep} 
{\topsep} 
{\itshape} 
{} 
{\bfseries} 
{} 
{\newline} 
{} 
\theoremstyle{theorem-style}
\newtheorem{theorem}{Theorem}
\newtheorem{proposition}{Proposition}
\newtheorem{corollary}{Corollary}
\newtheorem{lemma}{Lemma}

\newtheoremstyle{definition-style}
{\topsep} 
{\topsep} 
{} 
{} 
{\bfseries} 
{} 
{\newline} 
{} 
\theoremstyle{definition-style}

\newtheorem{remark}{Remark}

\usepackage{titletoc}
\usepackage{tocloft}
\usepackage[prependcaption,textsize=small,color=gray!40]{todonotes} 

\newif\ifcomments
\commentstrue
\ifcomments\newcommand{\comments}[1]{#1}\else\newcommand{\comments}[1]{}\fi

\newcommand{\beginsupplementary}{%

  \renewcommand{\thesection}{S\arabic{section}}
  \renewcommand{\thesubsection}{\thesection.\arabic{subsection}}
  \renewcommand{\theHsection}{S\arabic{section}} 

  \setcounter{table}{0}
  \renewcommand{\thetable}{S\arabic{table}}
  \setcounter{figure}{0}
  \renewcommand{\thefigure}{S\arabic{figure}}
  \setcounter{section}{0}
  \renewcommand{\theequation}{S\arabic{equation}} 
  \renewcommand{\thetheorem}{S\arabic{theorem}} 
  \renewcommand{\thedefinition}{S\arabic{definition}} 
  \renewcommand{\thecorollary}{S\arabic{corollary}} 
  \renewcommand{\theproposition}{S\arabic{proposition}} 
  \renewcommand{\theremark}{S\arabic{remark}} 
  \renewcommand{\thelemma}{S\arabic{lemma}} 
  \renewcommand{\theexample}{S\arabic{example}} 
  \renewcommand{\thealgorithm}{S\arabic{algorithm}}
}

\usepackage{url}
\usepackage{xr}
\usepackage{hyperref}
\hypersetup{
  colorlinks,
  linkcolor={black},
  citecolor=citeblue,
  urlcolor=citeblue,
  breaklinks=true,   
  linktoc=all,
}
\usepackage[nameinlink]{cleveref}
\usepackage[numbered]{bookmark} 

\usepackage{amsmath}
\usepackage{amsfonts}
\usepackage{amssymb}
\usepackage{bm}
\usepackage{physics}

\newcommand{\shat}[1]{\vphantom{#1}\smash[t]{\hat{#1}}}

\newcommand{\N}{\mathbb{N}}

\newcommand{\R}{\mathbb{R}}

\newcommand{\Rd}{\mathbb{R}^d}
\newcommand{\Rnn}{\mathbb{R}^{n\times n}}

\newcommand{\linspan}[1]{\operatorname{span}\{{#1}\}}

\renewcommand{\top}{{\intercal}}

\DeclareSymbolFont{stmry}{U}{stmry}{m}{n}
\DeclareMathSymbol\obar\mathrel{stmry}{"3A}
\DeclareMathSymbol\otimes\mathrel{stmry}{"0F}
\DeclareMathSymbol\ominus\mathrel{stmry}{"17}
\makeatletter
\newcommand{\superimpose}[2]{
  {\ooalign{$#1\@firstoftwo#2$\cr\hfil$#1\@secondoftwo#2$\hfil\cr}}}
\makeatother

\makeatother

\newcommand{\Normal}{\mathcal{N}}
\newcommand{\GP}{\mathcal{GP}}

\DeclareMathOperator*{\argmin}{arg\,min}

\newcommand{\bigO}{\mathcal{O}}

\def\vd{{\bm{d}}}
\def\ve{{\bm{e}}}
\def\vf{{\bm{\mathrm{f}}}}

\def\vl{{\bm{l}}}

\def\vr{{\bm{r}}}
\def\vs{{\bm{s}}}

\def\vu{{\bm{\mathrm{u}}}}
\def\vv{{\bm{v}}}

\def\vx{{\bm{x}}}
\def\vy{{\bm{y}}}
\def\vz{{\bm{z}}}

\def\vmu{{\bm{\mu}}}

\def\vzero{{\bm{0}}}
\def\vone{{\bm{1}}}

\def\evy{{y}}

\def\mA{{\bm{A}}}
\def\mB{{\bm{B}}}
\def\mC{{\bm{C}}}
\def\mD{{\bm{D}}}

\def\mH{{\bm{H}}}
\def\mI{{\bm{I}}}

\def\mK{{\bm{K}}}
\def\mL{{\bm{L}}}

\def\mP{{\bm{P}}}
\def\mQ{{\bm{Q}}}

\def\mS{{\bm{S}}}

\def\mU{{\bm{U}}}

\def\mX{{\bm{X}}}

\def\mZ{{\bm{Z}}}

\def\mLambda{{\bm{\Lambda}}}
\def\mSigma{{\bm{\Sigma}}}

\def\mZero{{\bm{0}}}

\DeclareMathAlphabet{\mathsfit}{\encodingdefault}{\sfdefault}{m}{sl}
\SetMathAlphabet{\mathsfit}{bold}{\encodingdefault}{\sfdefault}{bx}{n}

\def\idxiter{i}

\def\txtCG{\mathrm{CG}}
\def\txtPCG{\mathrm{PCG}}

\def\inputspace{\mathcal{X}}
\def\outputspace{\mathcal{Y}}
\def\kernel{k}
\def\kernmat{\mK}
\def\traindata{\mX}

\def\labels{\vy}

\def\action{\vs}
\def\mActions{\mS}
\def\observ{\alpha}
\def\residual{\vr}

\def\searchdir{\vd}
\def\mSearchdir{\mD}
\def\searchdirsqnorm{\eta}

\def\qoi{{\vv_*}}
\def\qoimean{\vv}
\def\qoicov{\mSigma}

\def\kernmatapprox{\mQ}
\def\invapprox{\mC}

\def\cholfac{\mL}

\def\precond{\mP}

\def\inducingpoint{\vz}
\def\inducingpoints{\mZ}

\newcommand{\papertitle}{Posterior and Computational Uncertainty\\ in Gaussian Processes}
\newcommand{\authorinfo}{
	Jonathan Wenger$^{1,2}$ \qquad Geoff Pleiss$^{2}$ \vspace{1em}\\ {\bfseries Marvin Pf\" ortner$^{1}$ \qquad Philipp Hennig$^{1,3}$ \qquad John P. Cunningham$^{2}$}\vspace{0.5em}\\
	$^1$ University of T\" ubingen\\
	$^2$ Columbia University\\
	$^3$ Max Planck Institute for Intelligent Systems, T\" ubingen
}

\title{\papertitle}
\author{\authorinfo}

\begin{document}

\maketitle

\begin{abstract}
	Gaussian processes scale prohibitively with the size of the dataset. In response, many
	approximation methods have been developed, which inevitably introduce approximation error. This
	additional source of uncertainty, due to limited computation, is entirely ignored when using the
	approximate posterior.
	Therefore in practice, GP models are often as much about the approximation method as they are about
	the data. Here, we develop a new class of methods that provides
	consistent estimation of the combined uncertainty arising from \emph{both} the finite
	number of data observed \emph{and} the finite amount of computation expended. The most
	common GP approximations map to an instance in this class, such as
	methods based on the Cholesky factorization, conjugate gradients, and inducing
	points. For any method in this class, we prove (i) convergence of its posterior mean in
	the associated RKHS, (ii) decomposability of its combined posterior covariance into mathematical
	and computational covariances, and (iii) that the combined variance is a tight
	worst-case bound for the squared error between the method's posterior mean and the latent
	function. Finally, we
	empirically demonstrate the consequences of ignoring computational uncertainty and show how
	implicitly modeling it improves generalization performance on benchmark datasets.
\end{abstract}

\section{Introduction}

Gaussian processes (GPs) are an expressive probabilistic model class, but their
prohibitive scaling necessitates approximation \cite{Rasmussen2006}. A range of approximations based on kernel
\cite{Zhu1997,Trecate1999,Rahimi2007,Wilson2015a,Wilson2015,Yang2015,Izmailov2018,Evans2018,Zandieh2020} or precision matrix
\cite{Vecchia1988,Datta2016,Katzfuss2021,Schaefer2021} estimates,
inducing point methods
\cite{Smola2000,Williams2001,Drineas2005,Seeger2003,Snelson2005,QuinoneroCandela2005,Titsias2009,Hensman2013}, and iterative solvers \cite{Gibbs1997,Murray2009,Cutajar2016,Gardner2018,Wang2019,Artemev2021,Wenger2022} have been proposed.
These methods all use an affordable amount of computation to obtain an approximation of the
\emph{mathematical} posterior, which exists theoretically but cannot be accessed given 
limited computational resources. The approximate posterior is then used as a direct replacement of
the mathematical posterior in downstream applications. Doing so, however, completely ignores the
fact that we only expended a limited amount of compute. By analogy to the typical GP operation, where \emph{limited data} induces modeling error captured by \emph{mathematical uncertainty}, our work is motivated by the fact that \emph{limited computation} induces approximation error that
must be captured by \emph{computational uncertainty}.

Here, we introduce IterGP, a class of methods
which return a \emph{combined uncertainty} that is the sum of mathematical and computational
uncertainty. \Cref{fig:main-illustration} illustrates the difference between ignoring computational
uncertainty and explicitly modeling it. We perform GP regression using a
Mat\'ern(\(\frac{3}{2}\)) kernel on a toy dataset and compare SVGP
(\colordot{mplred}) \cite{Hensman2013} to its analog in our framework, IterGP-PI
(\colordot{mplgreen} + \colordot{mplblue}), for a fixed set of inducing points.
The computational shortcuts of inducing point methods can lead to unavoidable biases in their posterior
mean and covariance \cite{Bauer2016,Huggins2019}. As \Cref{fig:main-illustration} illustrates, SVGP may
underestimate the marginal variance where inducing points do not coincide with datapoints. In
contrast, IterGP is guaranteed to overestimate the mathematical uncertainty -- with the difference
precisely given by the computational uncertainty (\colordot{mplgreen}).
Additionally, the computational uncertainty is a worst-case bound (\textcolor{mplgreen}{---}) on the
error of the approximate posterior mean. 
\begin{figure}
	\includegraphics[width=\textwidth]{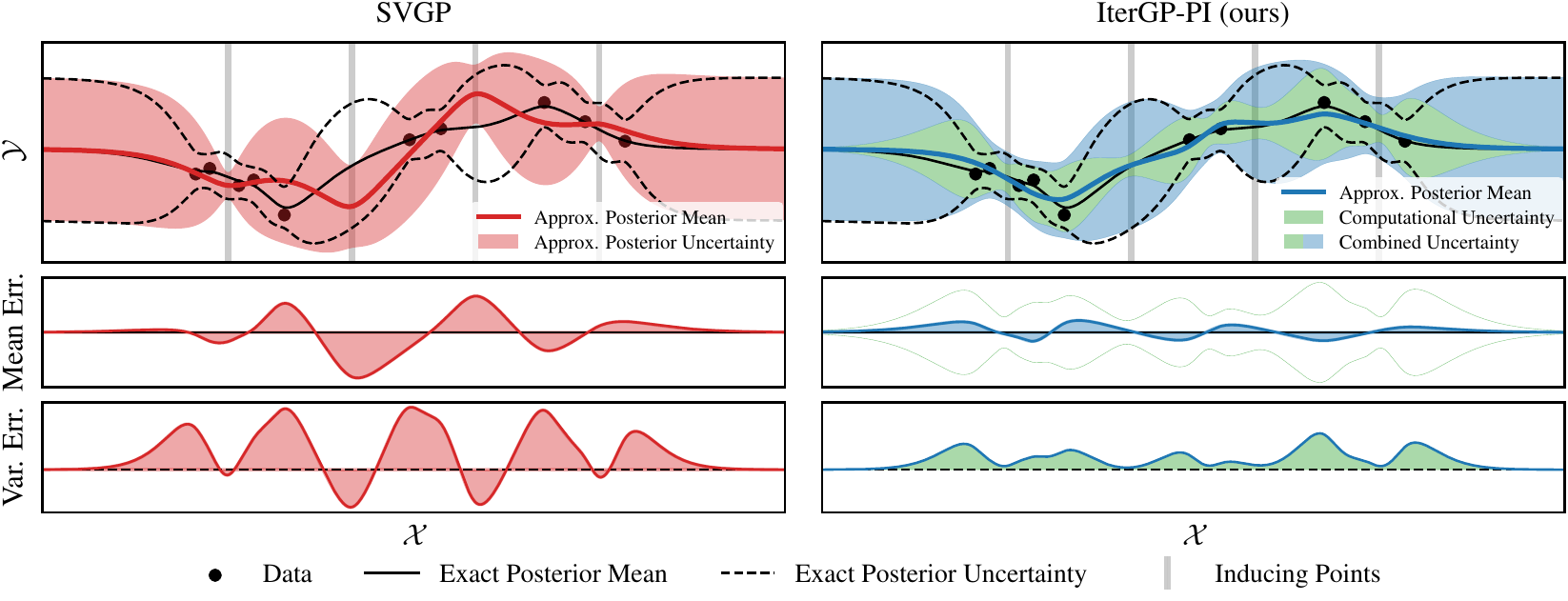}
	\caption{\emph{Modeling computational uncertainty improves GP approximation.}}
	\label{fig:main-illustration}
	\vspace{-0.65em}
\end{figure}

To be clear, this overestimation is desirable: IterGP is not a typical approximation in the sense
that its combined posterior attempts to approximate the mathematical posterior.  Rather, IterGP
recognizes that the mathematical posterior exists, but we do not have access to it, given
computational constraints. Finite compute is as true a source of posterior uncertainty
as finite data. Taking this view seriously, the true goal of GPs in the limited compute regime
should in fact be to track combined uncertainty. This intuition motivates IterGP and is formally a
feature of our results. We show that, if you update your GP via computation, specifically
matrix-vector multiplication, then the combined uncertainty of the IterGP algorithm is precisely
the correct object to capture your belief
(\Cref{thm:worst-case-error}) -- in the same way the mathematical posterior is the correct object given finite data and unlimited computation.

Formally, IterGP is a probabilistic numerical method \cite{Hennig2015a,Cockayne2019a,Oates2019,Hennig2022}. 
It treats the (unknown) representer weights as a latent variable with a prior belief
that, when marginalized out, corresponds to a GP prior conditioned on no data. We then use a
computational primitive (matrix-vector multiplication) that corresponds to tractable Bayesian
updates on the representer weight distribution; that is, conditioning on computations performed on
the data.
The resulting belief can then be marginalized out to obtain a closed form, tractable expression for
the combined -- mathematical plus computational -- uncertainty. This uncertainty quantification can
be done \emph{exactly} in quadratic time and linear space complexity.

Our framework admits three key theoretical properties. First, common GP approximations such as the
partial Cholesky, the method of conjugate gradients and inducing point methods (e.g. 
SVGP) map to a corresponding IterGP instance. Therefore,  these approaches can either be
directly extended or modified to properly account for computational uncertainty. Second, the
approximate posterior mean of any method in our proposed class converges to the mathematical posterior mean in RKHS norm in at most \(n\) steps, where the convergence rate is determined by the choice of method (\Cref{thm:convergence-rkhs}).
Third, the combined uncertainty is a tight worst case bound on the relative error between the
approximate posterior mean and the latent function (\Cref{thm:worst-case-error}). To the best of our
knowledge no existing GP approximation has this last property; an analoguous guarantee only holds
for exact GPs \cite[Sec.~3.4]{Kanagawa2018}.

\paragraph{Contribution}
This work introduces IterGP, which defines a new class of GP approximations that accounts for computational uncertainty arising from limited computation. Some IterGP instances extend classic methods with improved uncertainty quantification (\Cref{tab:gp-approx-methods}). For any method in this class, we prove that the approximate posterior mean converges to the mathematical posterior mean (\Cref{thm:convergence-rkhs}) and that the combined uncertainty is a tight worst-case bound on the relative distance to the latent function one is trying to learn (\Cref{thm:worst-case-error}, \Cref{cor:worst-case-error}). We demonstrate empirically that modeling computational uncertainty can either save computation or improve generalization on a set of regression benchmark datasets. In conclusion, we show that it is possible to exactly quantify the inevitable error of GP approximations at quadratic cost by propagating said error to the posterior in the form of computational uncertainty.
\raggedbottom

\section{Computation-Aware Gaussian Process Inference}
\label{sec:probabilistic-iterative-framework}

We aim to learn a latent function \(h :\inputspace \to
\outputspace\)
from \(\inputspace \subseteq \R^d\) to \(\outputspace
\subseteq \R\) given a training dataset \(\traindata = \begin{pmatrix}\vx_1, \dots, \vx_n \end{pmatrix} \in
\R^{n \times d}\) of \(n\)
inputs
\(\vx_j
\in \R^d\) and corresponding outputs \(\labels = \begin{pmatrix} \evy_1, \dots, \evy_n \end{pmatrix}^{\top} \in \R^n\).

\textbf{Gaussian Processes}
A stochastic process \(f \sim \GP(\mu, k)\) with mean function \(\mu : \Rd \to
\R\) and kernel \(k:\Rd \times \Rd \to \R\) is called a
\emph{Gaussian process} (GP) if
any collection of function values \(\vf = (f(\vx_1), \dots, f(\vx_n))^\top \sim
\Normal(\vmu, \kernmat)\) is jointly Gaussian with \(\vmu_j = \mu(\vx_j)\) and
\(\kernmat_{ij} =k(\vx_i,\vx_j)\). Assuming observation noise \(\labels \mid \vf \sim
\Normal(\vf, \sigma^2 \mI)\), the
posterior distribution at test inputs \(\mX_\diamond\) is given by \(\vf_\diamond \sim
\Normal(\mu_*(\mX_\diamond), k_*(\mX_\diamond, \mX_\diamond))\) where the posterior mean and
covariance functions are given by
\begin{equation}
	\begin{aligned}
		\label{eqn:posterior-mean-covariance}
		\mu_*(\cdot) = \mu(\cdot) + k(\cdot, \traindata) \overbracket[0.14ex]{\shat{\kernmat}^{-1} (\labels - \vmu)}^{\qoi}, \quad \text{and} \quad
		k_*(\cdot, \cdot) = k(\cdot, \cdot) - k(\cdot, \traindata)\shat{\kernmat}^{-1} k(\traindata,
		\cdot)
	\end{aligned}
\end{equation}
where \(\shat{\kernmat} \coloneqq \kernmat + \sigma^2 \mI \in \Rnn\).
Computing the \emph{representer weights} \(\qoi = \shat{\kernmat}^{-1} (\labels - \vmu)\) exactly (as well as the
posterior variance) is prohibitive given our limited computational budget.

\textbf{Learning Representer Weights} Consider the conditional distribution of the latent GP given its
representer weights:
%
\begin{equation}
	p( \vf_\diamond \mid \qoi ) =  \Normal( \mu(\mX_\diamond) + k(\mX_\diamond, \traindata) \qoi , \:\: k_*(\mX_\diamond, \mX_\diamond)
	).
\end{equation}
When $\qoi$ is known exactly, we recover \cref{eqn:posterior-mean-covariance}.
However, if we instead treat $\qoi$ as a random variable with the prior
\( p(\qoi) = \Normal(\qoi; \vzero, \shat{\mK}^{-1}) \),
then the resulting marginal \( \int p(\vf_\diamond \mid \qoi) p(\qoi) \,d \qoi \) recovers the GP
prior $\Normal(\mu(\mX_\diamond), k(\mX_\diamond, \mX_\diamond))$.
Our goal is to update this prior by iteratively applying the tractable computational primitive
(i.e. matrix-vector multiplies).
More specifically,
each iteration conditions the current belief distribution $p(\qoi) = \Normal( \qoi; \qoimean_{\idxiter - 1}, \qoicov_{\idxiter-1} )$
on a one-dimensional projection of the current \emph{residual}
\(\residual_{\idxiter-1} = (\labels - \vmu) - \shat{\kernmat} \qoimean_{\idxiter -1}\),
where the projection is defined by an arbitrary vector \(\action_\idxiter\):
%
\begin{equation}
	\observ_\idxiter \coloneqq \action_\idxiter^\top\residual_{\idxiter-1} =
	\action_\idxiter^\top((\labels - \vmu) - \shat{\kernmat} \qoimean_{\idxiter -1}) =
	\action_\idxiter^\top\shat{\kernmat}(\qoi - \qoimean_{\idxiter-1}).
	\label{eqn:information-op}
\end{equation}
The choice of \emph{actions} \(\action_\idxiter\),
which intuitively weight the approximation error of selected datapoints,
defines different instances of our IterGP framework.
Computing \cref{eqn:information-op} requires a single matrix-vector multiplication.
After computing \( \observ_\idxiter \), we can perform an exact Bayesian update of
$p(\qoi)$ via linear Gaussian identities.
The updated $p(\qoi)$ (conditioned on \( \observ_\idxiter \))
is \(\Normal(\qoi \mid \qoimean_\idxiter, \qoicov_\idxiter)\), with
\begin{align}
	\label{eqn:posterior-representer-weights}
	\qoimean_\idxiter & = \qoimean_{\idxiter-1} + \underbracket[0.14ex]{\qoicov_{\idxiter-1} \shat{\kernmat}\action_{\idxiter}}_{\eqqcolon \searchdir_\idxiter}(\underbracket[0.14ex]{\action_{\idxiter}^\top \shat{\kernmat} \qoicov_{\idxiter-1} \shat{\kernmat}
		\action_{\idxiter}}_{\eqqcolon \searchdirsqnorm_\idxiter})^{-1} \underbracket[0.14ex]{\action_{\idxiter}^\top \shat{\kernmat}(\qoi - \qoimean_{\idxiter-1})}_{= \observ_\idxiter}= \invapprox_\idxiter(\labels - \vmu)                                         \\
	\qoicov_\idxiter  & = \qoicov_{\idxiter-1} - \underbracket[0.14ex]{\qoicov_{\idxiter-1} \shat{\kernmat}\action_{\idxiter}}_{= \searchdir_\idxiter}(\underbracket[0.14ex]{\action_{\idxiter}^\top \shat{\kernmat} \qoicov_{\idxiter-1} \shat{\kernmat}
		\action_{\idxiter}}_{= \searchdirsqnorm_\idxiter})^{-1}\underbracket[0.14ex]{\action_{\idxiter}^\top \shat{\kernmat}\qoicov_{\idxiter-1}}_{=\searchdir_\idxiter^\top}= \shat{\kernmat}^{-1} - \invapprox_\idxiter.
\end{align}
where  \(\invapprox_\idxiter  \coloneqq \sum_{j=1}^\idxiter \frac{1}{\searchdirsqnorm_j}
\searchdir_j \searchdir_j^\top = \mActions_\idxiter (\mActions_\idxiter^\top \shat{\kernmat} \mActions_\idxiter)^{-1} \mActions_\idxiter^\top\) is a
rank-\(\idxiter\) matrix
(see \Cref{prop:problinsolve-batch-posterior} for details).
With each computation, the uncertainty about the representer weights contracts as
\(\invapprox_\idxiter \to \shat{\kernmat}^{-1} = \qoicov_0\) as \(\idxiter \to n\).
After \(n\) iterations, \(\invapprox_n = \shat\kernmat^{-1}\),
meaning we fully recovered the representer weights with zero uncertainty.
The consistent estimate for the representer weights is consequently \(\qoimean_\idxiter =
\invapprox_\idxiter (\labels - \vmu)\).

\textbf{Combining Mathematical and Computational Uncertainty}
We now have a belief \(p(\qoi) = \Normal(\qoi; \qoimean_{\idxiter}, \qoicov_{\idxiter})\) about the representer weights reflecting the expended computation. To account for this computational uncertainty, we treat the
representer weights as a latent variable of the mathematical posterior by reparameterizing 
\(p(\vf_\diamond \mid \labels) = p(\vf_\diamond \mid \qoi)\) and then marginalizing.
The resulting marginal considers all possible representer weights which would have resulted in the
same computational observations and therefore \emph{implicitly} adds the uncertainty coming
from \emph{the computation itself}. Since the posterior mean of a GP is a linear function of the
representer weights, the marginal distribution is given by \(p(\vf_\diamond) = \int p(\vf_\diamond \mid \qoi)
p(\qoi)\,d\qoi = \Normal(\vf_\diamond; \mu_i(\mX_\diamond), k_i(\mX_\diamond, \mX_\diamond))\), where
\begin{equation}
	\label{eqn:marginal_gp_posterior}
	\begin{aligned}
		\mu_i(\cdot)    & = \mu(\cdot) + \kernel(\cdot, \traindata) \qoimean_\idxiter   \\  
		k_i(\cdot, \cdot) & = \underbracket[0.14ex]{\kernel(\cdot, \cdot) - \kernel(\cdot, \traindata)\shat{\kernmat}^{-1}\kernel(\traindata, \cdot)}_{\textup{mathematical uncertainty \colordot{mplblue}}} + \underbracket[0.14ex]{\kernel(\cdot, \traindata)\qoicov_\idxiter\kernel(\traindata, \cdot)}_{\mathclap{\textup{computational uncertainty \colordot{mplgreen}}}} = \underbracket[0.14ex]{\kernel(\cdot, \cdot) - \kernel(\cdot, \traindata)\invapprox_\idxiter\kernel(\traindata, \cdot)}_{\textup{combined uncertainty \colordot{mplpurple}}}
	\end{aligned}
\end{equation}
since \(\qoicov_\idxiter = \shat{\mK}^{-1} - \invapprox_\idxiter\).\footnote{While we derive the combined posterior from a probabilistic numerics perspective, we can alternatively interpret \cref{eqn:marginal_gp_posterior} as the GP prior \(f\) conditioned on linearly transformed data \(\mActions_\idxiter^\top \labels \mid \vf \sim \Normal(\mActions_\idxiter^\top \vf, \sigma^2 \mActions_\idxiter^\top \mActions_\idxiter)\).}
As we perform more computation, the computational uncertainty reduces and we approach the
mathematical uncertainty. We note that, while the \emph{individual} terms are computationally
prohibitive, the \emph{combined} uncertainty can be evaluated cheaply since the approximate
precision matrix \(\invapprox_\idxiter\) is of low rank. \Cref{fig:uncertainty-decomposition} illustrates that
computational uncertainty is large where there are data and we have not targeted computation yet.
Different methods from our proposed class target computation in different parts of the input space.
Where there is no data the prior is a good approximation of the posterior and therefore
computational uncertainty is low.

\begin{figure}
	\includegraphics[width=\textwidth]{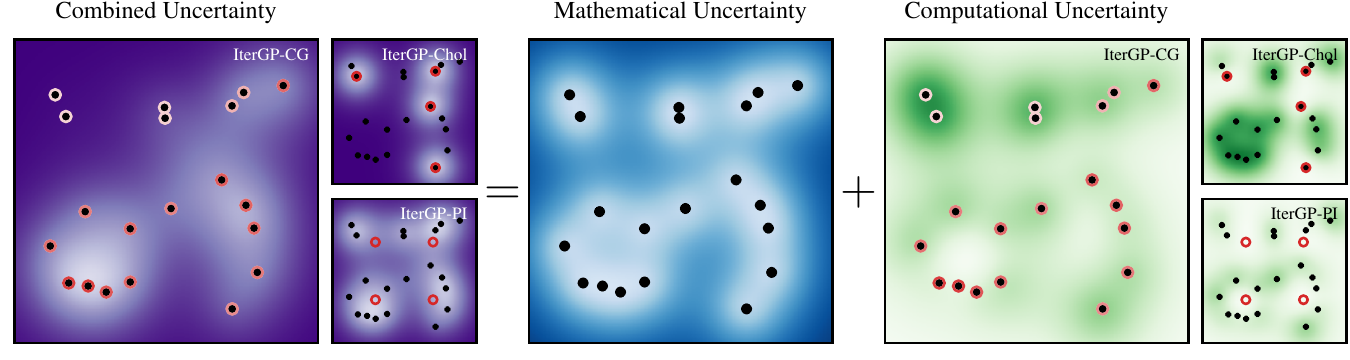}
	\caption{\emph{Decomposition of the combined uncertainty.} The combined uncertainty (\colordot{mplpurple}) output by IterGP
		decomposes into the mathematical uncertainty (\colordot{mplblue}) and computational uncertainty
		(\colordot{mplgreen}). After \(\idxiter=4\) iterations of \Cref{alg:itergp} computational
		uncertainty is small in parts of the input space where there either is no data
		(\filledcolordot[0.5ex]{black}) or computation was ``targeted'' (\colorcircle[0.5ex]{mplred}). Which datapoints
		are targeted in each iteration and to what degree is defined by the magnitude of the action vector
		elements \((\action_\idxiter)_j\). Different instances of IterGP either reduce computational
		uncertainty locally (e.g. IterGP-Chol, IterGP-PI) or globally (e.g. IterGP-CG). After \(n\) iterations the mathematical uncertainty is
		recovered.}
	\label{fig:uncertainty-decomposition}
\end{figure}

\Cref{alg:itergp} computes an estimate of the representer weights \(\qoimean_\idxiter\) and
the rank-\(\idxiter\) precision matrix approximation \(\invapprox_\idxiter\). A specific instance
of IterGP is defined by a sequence of actions \(\action_\idxiter\).  To gain an intuition for how
\Cref{alg:itergp} operates, it helps to interpret it as targeting a given computational
budget towards certain datapoints defined by \(\action_\idxiter\). Near datapoints \(\vx_j\) that
are not targeted, i.e. \((\action_\idxiter)_j = 0\), computational uncertainty remains unchanged.
In fact, datapoints \((\vx_j, \evy_j)\) that are never targeted up to iteration \(\idxiter\) are
not needed to compute \(\GP(\mu_\idxiter, k_\idxiter)\), meaning that \Cref{alg:itergp} is
\emph{inherently online} and we can \emph{observe data sequentially} without having to restart the algorithm
(see \Cref{thm:online-inference}).

\begin{algorithm}[H]
	\caption{A Class of Computation-Aware Iterative Methods for GP Approximation\label{alg:itergp}}
	\small
	\textbf{Input:} prior mean function \(\mu\), prior covariance function / kernel \(\kernel\), training inputs \(\traindata\), labels \(\labels\)\\
\textbf{Output:} (combined) GP posterior \(\GP(\mu_{\idxiter}, k_{\idxiter})\)
\begin{algorithmic}[1]
    \Procedure{\textsc{IterGP}}{$\mu, \kernel, \traindata, \labels$}
    \State \((\mu_0, k_0) \gets (\mu, k)\) \Comment{Initialize mean and covariance function with prior.}
    \State \(\vmu \gets \mu(\traindata)\) \Comment{Prior predictive mean.}
    \State \(\shat{\kernmat} \gets \kernel(\traindata, \traindata) + \sigma^2 \mI\) \Comment{Prior predictive kernel matrix.}
    \While{\textbf{not} \textsc{StoppingCriterion}()} \Comment{Stopping criterion.}
    \State \(\action_\idxiter \gets
    \textsc{Policy}()\) \Comment{Select action via policy (see \Cref{tab:gp-approx-methods} for examples).}
    \State \(\residual_{\idxiter - 1} \gets (\labels - \vmu) -  \shat{\kernmat} \qoimean_{\idxiter-1}\)
    \Comment{Predictive residual.}
    \State \(\observ_\idxiter \gets \action_\idxiter^\top \residual_{\idxiter-1}\)
    \Comment{Observation via information operator.}
    \State \(\searchdir_\idxiter \gets \textcolor{gray}{\qoicov_{\idxiter-1} \shat{\kernmat}
    \action_\idxiter =} \, (\mI - \invapprox_{\idxiter-1}\shat{\kernmat})\action_\idxiter\) \Comment{Search direction.}
    \State \(\searchdirsqnorm_\idxiter \gets \textcolor{gray}{\action_{\idxiter}^\top \shat{\kernmat} \qoicov_{\idxiter-1} \shat{\kernmat} \action_{\idxiter} =}\, \action_{\idxiter}^\top \shat{\kernmat} \searchdir_\idxiter\) \Comment{Normalization constant.}
    \State \(\invapprox_\idxiter \gets \invapprox_{\idxiter-1} + \frac{1}{\searchdirsqnorm_\idxiter}\searchdir_\idxiter \searchdir_\idxiter^\top\)
    \Comment{Precision matrix approximation \(\invapprox_\idxiter \approx \shat{\kernmat}^{-1}\).}
    \State \textcolor{gray}{\(\kernmatapprox_\idxiter \gets \kernmatapprox_{\idxiter-1} + \frac{1}{\searchdirsqnorm_\idxiter} \shat{\kernmat} \searchdir_\idxiter \searchdir_\idxiter^\top\shat{\kernmat}\)
    \Comment{Kernel matrix approximation \(\kernmatapprox_\idxiter \approx \shat{\kernmat}\).}}
    \State \(\qoimean_\idxiter \gets \qoimean_{\idxiter-1} + \frac{\alpha_\idxiter}{\searchdirsqnorm_\idxiter}\searchdir_\idxiter  \)
    \Comment{Representer weights estimate.}
    \State \(\qoicov_\idxiter \gets \qoicov_0 - \invapprox_\idxiter\)
    \Comment{Computational representer weights uncertainty.}
    \EndWhile
    \State \textcolor{gray}{\(p(\qoi) \gets \Normal(\qoi; \qoimean_\idxiter, \qoicov_\idxiter)\) \Comment{Belief about representer weights.}}
    \State \(\mu_{\idxiter}(\cdot) \gets \mu(\cdot) + \kernel(\cdot, \traindata) \qoimean_\idxiter\)\Comment{Approximate posterior mean function.}
    \State \(\kernel_{\idxiter}(\cdot, \cdot) \gets \kernel(\cdot, \cdot)	- \kernel(\cdot, \traindata) 
    \invapprox_\idxiter \kernel(\traindata, \cdot)\) \Comment{Combined uncertainty.}
    \State \Return \(\GP(\mu_{\idxiter}, k_{\idxiter})\)
    \EndProcedure
\end{algorithmic}
\end{algorithm}
\vspace{-1.5\baselineskip}%
{\hfill \scriptsize \textcolor{gray}{Greyed out quantities are \emph{not} needed to compute the combined posterior and are only included for clarity of exposition.} \par}

\subsection{Connection to Other GP Approximation Methods}
\label{sec:connection-gp-approx}

IterGP extends the most commonly used GP approximations to include computational uncertainty, with
at most quadratic cost (see \Cref{tab:gp-approx-methods} for a summary and \Cref{fig:uncertainty-decomposition},
\Cref{fig:recovered-methods-illustration} for illustration).

\begin{table}
	\caption{\emph{Instances of \Cref{alg:itergp}, which map to commonly used
			GP approximations.}}
	\centering
	\small
	\begin{tabular}{lccccc}
    \toprule
    Method      & Actions \(\vs_\idxiter\)                                                & Classic Analog                       & Reference                                                      \\
    \midrule
    IterGP-Chol & \(\ve_\idxiter\)                                                        & (partial) Cholesky                   & \Cref{thm:itergp-cholesky}                                     \\
    IterGP-PBR  & \(\mathrm{ev}_\idxiter(\shat{\kernmat})\)                               & (partial) EVD / SVD                  & \Cref{thm:itergp-svd}                                          \\
    IterGP-CG   & \(\action_\idxiter^{\txtPCG}\) or \(\shat{\mP}^{-1}\residual_\idxiter\) & (preconditioned) CG                  & \Cref{thm:itergp-cg} and \Cref{cor:itergp-cg-gradient-actions} \\
    IterGP-PI   & \(k(\traindata, \vz_\idxiter)\)                                         & \(\approx\) Nyström (SoR, DTC), SVGP & \Cref{sec:connection-gp-approx} and \Cref{thm:itergp-nystroem} \\

    \bottomrule
\end{tabular}
	\label{tab:gp-approx-methods}
\end{table}

\textbf{Cholesky Decomposition}
The (partial) Cholesky decomposition iteratively chooses datapoints or pivots \(\vx_\idxiter\)
based on a given ordering. The resulting Cholesky factor is lower triangular and
increases in rank each iteration,
and a well-chosen ordering achieves fast convergence (cf. \cite[Thm.~2.3]{Schaefer2021a}). If one chooses
standard unit vectors \(\ve_\idxiter\) as actions corresponding to the selected datapoint per iteration, then
\Cref{alg:itergp} recovers the partial Cholesky factorization exactly (\Cref{thm:itergp-cholesky}).

\textbf{Conjugate Gradients}
CG \cite{Hestenes1952} with preconditioning for GP inference has become increasingly popular
\cite{Cunningham2008,Murray2009,Cutajar2016,Gardner2018,Wang2019,Artemev2021,Potapczynski2021,Wenger2022}. \Cref{alg:itergp} recovers preconditioned CG exactly, if we choose
either preconditioned conjugate gradients or residuals as actions (see \Cref{thm:itergp-cg} and
\Cref{cor:itergp-cg-gradient-actions}). In fact, \Cref{alg:itergp} can even construct its own
diagonal-plus-low-rank preconditioner by first running a few iterations with an arbitrary policy
and then using the byproducts of these iterations for the preconditioner. For example, if we run
IterGP-Chol initially, we can construct an incomplete Cholesky preconditioner for subsequent CG
iterations.

\textbf{Inducing Point Methods}
Inducing point methods, such as variants of the Nyström approximation \cite{Williams2001}, i.e. subset of regressors (SoR) \cite{Silverman1985,Smola2000} and deterministic training conditional (DTC) \cite{Csato2002,Seeger2003}, as well as SVGP \cite{Hensman2013} share a posterior mean, which by
\Cref{thm:itergp-nystroem} takes the form
\begin{align}
	\label{eqn:posterior-mean-svgp}
	\mu_{\mathrm{SVGP}}(\cdot) &= q(\cdot, \traindata) \kernmat_{\traindata
	\inducingpoints}(\kernmat_{\inducingpoints \traindata} (q(\traindata, \traindata) + \sigma^2
	\mI)\kernmat_{\traindata\inducingpoints })^{-1} \kernmat_{\inducingpoints \traindata}(\labels - \vmu)
\intertext{where \(\inducingpoints \in \R^{n \times \idxiter}\) is a set of inducing points and \(q(\cdot,
\cdot) = \kernel(\cdot, \inducingpoints) \kernmat_{\inducingpoints \inducingpoints}^{-1} \kernel(\inducingpoints, \cdot)\). These
approximations also have very closely related posterior covariance functions
\cite{QuinoneroCandela2005,Wild2021}. If we choose actions \(\action_\idxiter = k(\traindata, \vz_\idxiter)\),
by \Cref{prop:problinsolve-batch-posterior}, \Cref{alg:itergp} returns a posterior mean given by}
	\label{eqn:posterior-mean-pseudoinput}
	\mu_\idxiter(\cdot) &= \kernel(\cdot, \traindata)\kernmat_{\traindata \inducingpoints}(\underbracket[0.14ex]{\kernmat_{\inducingpoints\traindata }(\kernel(\traindata, \traindata) + \sigma^2\mI)\kernmat_{\traindata \inducingpoints}}_{\mathclap{\textup{Gram matrix } \mActions_\idxiter^\top \shat{\kernmat} \qoicov_0 \shat{\kernmat} \mActions_\idxiter}})^{-1}\kernmat_{\inducingpoints \traindata}(\labels-\vmu).
\end{align}
Choosing such actions, given by kernel functions \(k(\cdot, \inducingpoint_\idxiter)\) centered at inducing points \(\inducingpoint_\idxiter\), reduces computational uncertainty in regions close to inducing points (see
IterGP-PI in \Cref{fig:uncertainty-decomposition}), where closeness is determined by the kernel. Comparing SVGP's and IterGP-PI's posterior mean provides a probabilistic numerical perspective on why even for small KL-divergence between the approximating distribution of SVGP and the true posterior, the mean estimate can be far from the true mean \cite[Prop.~3.1]{Huggins2019}. As outlined in \Cref{sec:probabilistic-iterative-framework}, \cref{eqn:posterior-mean-pseudoinput} is a Bayesian update on the initially unknown representer weights \(\qoi = \shat{\kernmat}^{-1}(\labels - \vmu)\). The Gram matrix in \cref{eqn:posterior-mean-pseudoinput} describes how surprising the computational observations \(\kernmat_{\inducingpoints \traindata}(\labels-\vmu)=\mActions_\idxiter^\top (\labels-\vmu) = \mActions_\idxiter^\top \hat{\kernmat}\qoi\) of the representer weights should be,  given the prior uncertainty \(\qoicov_0\) about them. SVGP uses a similar form for the posterior mean (c.f. \eqref{eqn:posterior-mean-svgp} and \eqref{eqn:posterior-mean-pseudoinput}), but the Gram matrix is ``smaller'' since \(q(\traindata, \traindata) \preceq k(\traindata, \traindata)\). This can be interpreted as inducing point methods being overconfident in their update of the representer weight estimates to achieve linear time complexity. As the inducing points approach the data points the two posterior mean functions \(\mu_{\mathrm{SVGP}}\) and \(\mu_\idxiter\) become closer and are equivalent if the inducing points equal the training data.

\subsection{The Cost of Computational Uncertainty}

Quantifying combined uncertainty has greater cost than linear time GP approximations such as
inducing point methods, due to its use of matrix-vector multiplication as the computational
operation to condition on the data. \Cref{alg:itergp} in its most general form performs three
matrix-vector products per iteration resulting in a quadratic time complexity
\(\bigO(n^2 \idxiter)\) overall for \(\idxiter\) iterations.
In this sense, \Cref{alg:itergp} represents a middle ground between the mathematical
posterior---which incurs a cubic time complexity---%
and \(\bigO(ni^2)\) approximations---which can only estimate their computational error
through potentially loose theoretical bounds which may \citep[e.g.][]{Titsias2009,Hensman2013,Burt2019} or may not be
computable in less than \(\bigO(n^3)\) \cite{Rahimi2007,Schaefer2021a}.
At any point during a run of \Cref{alg:itergp}, computing the predictive mean on
\(n_\diamond\) new data points has cost \(\bigO(n_\diamond n)\), while the marginal predictive
(co-)variance can be evaluated in \(\bigO(n_\diamond n \idxiter)\) since \(\invapprox_\idxiter\) is of
rank \(\idxiter\).
Additionally, using Matheron's rule \cite{Journel1976,Wilson2020,Wilson2020a}, sampling from the approximate
posterior at $n_\diamond$ evaluation points also only requires \( \bigO(n_\diamond n \idxiter)
\) computation
(assuming we can sample from the prior---see \Cref{suppsec:efficient-sampling}).
The objects required to make predictions and draw samples are the vector $\qoimean_\idxiter$ and
low rank matrix \(\invapprox_\idxiter\)
which both require \( \bigO(n \idxiter) \) memory.
Finally, the memory cost of \Cref{alg:itergp} is only linear in $n$,
since matrix multiplication \(\vv \mapsto \shat{\kernmat}\vv\)
can be computed without explicitly forming \(\shat\kernmat\) 
\cite{Charlier2021}.


\subsection{Related Work}
GP inference based on matrix-vector multiplies, particularly CG
\cite{Hestenes1952}, has become popular recently \cite{Cunningham2008,Murray2009,Wilson2015a,Cutajar2016,Gardner2018,Wang2019,Artemev2021,Wenger2022}. Advances in
specialized hardware has boosted their scalability without excessive memory
footprint \cite{Charlier2021,Wang2019}. Such iterative methods typically rely on preconditioning, which
has been shown to significantly improve their performance \cite{Cutajar2016,Gardner2018,Wenger2022}. Our method
generalizes CG in this setting and thus retains the same benefits.
At its core \Cref{alg:itergp} employs a (Bayesian) probabilistic numerical method
\cite{Hennig2015a,Cockayne2019a,Oates2019,Hennig2022}, more specifically a probabilistic linear solver (PLS)
\cite{Hennig2015,Bartels2019,Cockayne2019,Cockayne2020,Wenger2020,Reid2022}
applied to the linear system \(\shat{\kernmat} \qoi = \labels\). The fact that a PLS using
CG actions can recover CG in its posterior mean
was observed previously \cite{Hennig2015,Cockayne2019,Wenger2020}. Here, we extend this result to residual actions
and preconditioning.  Further, we also demonstrate the connection to the Cholesky and singular
value decompositions. For randomized actions, the PLS as part of \Cref{alg:itergp} also
recovers the randomized Kaczmarz method in its posterior mean \cite{Kaczmarz1937,Strohmer2009,Gower2015,Gower2016}.
Employing a PLS for GP approximation by updating beliefs over the kernel and precision
matrix was suggested previously \cite{Bartels2020,Wenger2020}. Our
work differs in that it updates a belief over the representer weights, as opposed to
the kernel function or matrix, considers more general projections than just conjugate residuals,
and, most importantly, provides a theoretically motivated combined posterior which can be computed
exactly.







\section{Theoretical Analysis}

The main goals of our theoretical analysis will be to prove
\begin{enumerate}[itemsep=0pt,topsep=0pt,label=(\alph*)]
	\item \emph{convergence of IterGP's posterior mean} 
	in norm (\Cref{thm:convergence-rkhs}) and pointwise (\Cref{cor:worst-case-error})
\end{enumerate}
and to provide rigorous justification for the combined and computational uncertainty. Importantly, the
\begin{enumerate}[itemsep=0pt,topsep=0pt,label=(\alph*),start=2]
	\item \emph{combined uncertainty is a tight worst-case bound on the relative distance to all potential latent functions} consistent with our
	(computational) observations (\Cref{thm:worst-case-error}).
\end{enumerate}
We will demonstrate a similar interpretation of the computational uncertainty as a bound on the relative error to the mathematical posterior mean (see \cref{eqn:computational-uncertainty-worst-case,eqn:pointwise-convergence-to-posterior-mean}).

\subsection{Estimation of Representer Weights}
At the heart of \Cref{alg:itergp} is
a probabilistic linear solver \cite{Hennig2015,Cockayne2019,Bartels2019,Wenger2020} iteratively updating a belief about the
representer weights. It constructs an expanding subspace \(\linspan{\action_1, \dots, \action_\idxiter} =
\linspan{\searchdir_1, \dots, \searchdir_\idxiter}\) spanned by the actions in which the inverse \(\shat{\kernmat}^{-1}\)
is perfectly identified. Each step \(\searchdir_\idxiter\) expanding this explored subspace is
\(\shat{\kernmat}\)-orthogonal to the previous ones.

\begin{restatable}[Conjugate Direction Method]{proposition}{conjugatedirectionmethod}
	\label{prop:conjugate-direction-method}
	Let the actions \(\vs_\idxiter\) of \Cref{alg:itergp} be linearly independent. Then
	\Cref{alg:itergp} is a conjugate direction method, i.e. it holds that \(\searchdir_i^\top
	\shat{\kernmat} \searchdir_j = 0\) for all \(i \neq j\).
\end{restatable}
\begin{proof}
	Without loss of generality assume \(i > j\). Then the result follows directly from
	\Cref{lem:geometric-properties}.
\end{proof}

Geometrically, \Cref{alg:itergp} iteratively projects the representer weights onto the
expanding subspace \(\linspan{\mActions_\idxiter}\) with respect to \(\langle \cdot, \cdot
\rangle_{\shat{\kernmat}}\). We can use this intuition to understand the convergence of the
representer weights estimate. The relative error \(\rho(\idxiter)\) at iteration
\(\idxiter\) is given by how small the ``angle'' between this subspace and the representer
weights vector is.

\begin{restatable}[Relative Error Bound for the Representer Weights]{proposition}{representerweightserror}
	\label{prop:representer-weights-error}
	For any choice of actions a relative error bound \(\rho(\idxiter)\), s.t.
	\(\norm{\qoi - \qoimean_\idxiter}_{\shat{\mK}}\leq \rho(\idxiter)
	\norm{\qoi}_{\shat{\mK}}\) is given by
	\begin{equation}
		\label{eqn:representer-weights-error}
		\rho(\idxiter) = (\bar{\vv}_*^\top
		\underbracket[0.14ex]{(\mI - \invapprox_\idxiter \shat{\kernmat})}_{\clap{\scriptsize projection onto
		\(\linspan{\mActions_\idxiter}^{\perp_{\shat{\kernmat}}}\)}}
		\bar{\vv}_*)^{\frac{1}{2}} \leq \lambda_{\max}(\mI - \invapprox_\idxiter \shat{\kernmat}) \leq
		1
	\end{equation}
	where \(\bar{\vv}_* = \qoi/\norm{\qoi}_{\shat{\kernmat}}\). If
	the actions \(\{\action_\idxiter\}_{\idxiter=1}^n\) are linearly independent, then
	\(\rho(\idxiter) \leq 1 - \delta_{n = \idxiter}\).
\end{restatable}
\begin{proof}
	See \Cref{suppsec:approx-representer-weights}.
\end{proof}

\Cref{prop:representer-weights-error} guarantees convergence in at most \(n\) iterations, if the actions are chosen to
be linearly independent, since \(\invapprox_\idxiter \shat{\kernmat}\) is a
\(\shat{\kernmat}\)-orthogonal projection onto \(\linspan{\mActions_\idxiter}\) (see
\Cref{lem:geometric-properties}). Therefore, if our finite computational budget is large enough,
we eventually recover the mathematical posterior. This is reflected by the contraction of the
posterior over the representer weights (see \Cref{prop:posterior-contraction}). The bound in
\Cref{prop:representer-weights-error} is tight without further assumptions on the actions, since there exists an adverserial sequence of actions such that the first \((n-1)\) are in
\(\linspan{\qoi}^{\perp_{\shat\kernmat}}\). Then the inverse is perfectly identified in
that subspace, but \(\qoimean_\idxiter = \invapprox_\idxiter \labels =	\invapprox_\idxiter
\shat{\kernmat} \qoi = \vzero\). In practice, one can derive tighter convergence bounds for
specific sequences of actions. For example, for randomized actions the bound depends on their
distribution \cite{Strohmer2009,Gower2015}. If residuals \(\residual_\idxiter\) are chosen as actions, we
obtain
\begin{equation}
	\label{eqn:convergence-rate-cg}
	\textstyle\rho(\idxiter) = 2\left(\frac{\sqrt{\kappa} - 1}{\sqrt{\kappa} + 1}\right)^\idxiter \text{ or }
	\rho(\idxiter) = \left(\frac{\lambda_{n-\idxiter} - \lambda_1}{\lambda_{n-\idxiter} + \lambda_1}\right)
\end{equation}
since then \Cref{alg:itergp}'s estimate of the representer weights equals that of CG
(\Cref{cor:itergp-cg-gradient-actions}). Here \(\kappa\) is the condition number and \(\lambda_j\) the
eigenvalues of either (i) the kernel matrix \(\shat{\kernmat}\) if \(\vs_\idxiter =
\vr_\idxiter\),
\vspace{-0.5ex}
or (ii) the preconditioned kernel matrix
\(\shat{\precond}^{-\frac{1}{2}}\shat{\kernmat}\shat{\precond}^
{-\frac{\top}{2}}\) if \(\vs_\idxiter = \shat{\precond}^{-1}\vr_\idxiter\).

\subsection{Convergence in RKHS Norm of the Posterior Mean}
Having established convergence of the representer weights estimate, we can use this result to prove convergence in norm of IterGP's posterior mean to the mathematical posterior at the same rate.

\begin{restatable}[Convergence in RKHS Norm of the Posterior Mean Approximation]{theorem}{thmconvergencerkhs}
	\label{thm:convergence-rkhs}
	Let \(\mathcal{H}_k\) be the RKHS
	associated with kernel \(k(\cdot, \cdot)\), \(\sigma^2 > 0\) and let \(\mu_* - \mu \in
	\mathcal{H}_k\) be the unique
	solution to the regularized empirical risk minimization problem
	\begin{equation}
		\label{eqn:empirical-risk-minimization-problem}
		\textstyle\argmin_{f \in \mathcal{H}_k} \frac{1}{n} \big(\sum_{j=1}^n (f(\vx_j)
		-
		\evy_j + \mu(\vx_j))^2 + \sigma^2 \norm{f}_{\mathcal{H}_k}^2 \big)
	\end{equation}
	which is equivalent to the mathematical posterior mean up to shift by the prior \(\mu\) \citep[e.g.][Sec.~6.2]{Rasmussen2006}. Then for \(\idxiter \in
	\{0,
	\dots, n\}\) the posterior mean \(\mu_\idxiter(\cdot)\) computed by \Cref{alg:itergp}
	satisfies
	\begin{equation}
		\boxed{
		\norm{\mu_* - \mu_\idxiter}_{\mathcal{H}_k} \leq \rho(\idxiter) c(\sigma^2)
		\norm{\mu_* - \mu_0}_{\mathcal{H}_k}
		}
	\end{equation}
	where \(\mu_0 = \mu\) is the prior mean and the constant \(c(\sigma^2) =
	\sqrt{1 + \frac{\sigma^2}{\lambda_{\min}(\kernmat)}} \to 1\) as \(\sigma^2 \to 0\).
\end{restatable}
\begin{proof}
	See \Cref{suppsec:convergence-posterior-mean}.
\end{proof}
\Cref{thm:convergence-rkhs} gives a bound on the RKHS-norm error between the posterior mean \(\mu_\idxiter\) of IterGP and the mathematical posterior mean \(\mu_*\). If for the given prior kernel a bound on the RKHS-norm error \(\norm{h -  \mu_*}_{\mathcal{H}_k}\) between the latent function \(h\) and the  mathematical posterior mean \(\mu_*\) is known, \Cref{thm:convergence-rkhs} can be directly used to bound the RKHS-norm error between IterGP's posterior mean and the latent function \(h\) via the triangle inequality:  \(\norm{h - \mu_\idxiter}_{\mathcal{H}_k} \leq \underbracket[0.14ex]{\norm{h -  \mu_*}_{\mathcal{H}_k}}_{\to 0 \ \textup{ as } \ n \to \infty} + \underbracket[0.14ex]{\norm{\mu_* - \mu_\idxiter}_{\mathcal{H}_k}}_{\to 0 \ \textup{ as } \ \idxiter \to n}\).

\subsection{Combined and Computational Uncertainty as Worst Case Errors}

While \Cref{thm:convergence-rkhs} shows convergence in norm for IterGP's posterior mean, the convergence rate \(\rho(\idxiter)\) may contain expressions which cannot be evaluated at runtime with the limited
computation at our disposal. For example, for residual actions evaluating \cref{eqn:convergence-rate-cg} requires computation of the kernel matrix spectrum. However, the combined uncertainty of
IterGP is a tight bound on the \emph{pointwise} relative error to all possible latent
functions which would have resulted in the same computations.

\begin{restatable}[Combined and Computational Uncertainty as Worst Case Errors]{theorem}{thmworstcaseerror}
	\label{thm:worst-case-error}
	Let \(\sigma^2 \geq 0\) and let \(k_{\idxiter}(\cdot, \cdot) = k_*(\cdot, \cdot) +
	k_\idxiter^{\textup{comp}}(\cdot, \cdot)\)
	be the combined uncertainty computed by
	\Cref{alg:itergp}. Then, for any \(\vx \in
	\mathcal{X}\) (assuming \(\vx \notin \traindata\) if \(\sigma^2 > 0\)) we have
	\begin{align}
		    \sup_{g \in \mathcal{H}_{k^\sigma} : \norm{g}_{\mathcal{H}_{k^\sigma}} \leq 1}\hspace{1em} \overbracket[0.14ex]{\underbracket[0.14ex]{g(\vx) \mathbin{\textcolor{gray}{-}} \textcolor{gray}{\mu_*^g(\vx)}}_{\mathclap{\textup{error of math. post. mean \colordot{mplblue}}}} \hspace{0.5em} \textcolor{gray}{+}
			\hspace{0.5em}\underbracket[0.14ex]{\textcolor{gray}{\mu_*^g(\vx)} - \mu_\idxiter^g(\vx)}_{\textup{computational error \colordot{mplgreen}}}}^{\textup{error of approximate posterior mean \colordot{mplpurple}}}  \label{eqn:combined-uncertainty-worst-case} &= \sqrt{k_\idxiter(\vx, \vx) + \sigma^2}, \quad \text{and}                                           \\
		\sup_{g \in \mathcal{H}_{k^\sigma} : \norm{g}_{\mathcal{H}_{k^\sigma}} \leq 1}\underbracket[0.14ex]{\mu_*^g(\vx) - \mu_\idxiter^g(\vx)}_{\textup{computational error \colordot{mplgreen}}} &=\sqrt{k_\idxiter^{\textup{comp}}(\vx, \vx)}  \label{eqn:computational-uncertainty-worst-case}
	\end{align}
	where \(\mu_*^g(\cdot) = k(\cdot, \traindata) \shat{\kernmat}^{-1}g(\traindata)\) is the
	mathematical and \(\mu_\idxiter^g(\cdot) = k(\cdot, \traindata) \invapprox_\idxiter g(\traindata)\)
	IterGP's posterior mean for the latent function \(g \in \mathcal{H}_{k^\sigma}\). If
	\(\sigma^2=0\), then 
	the above also holds for \(\vx \in \traindata\).
\end{restatable}
\begin{proof}
	See \Cref{suppsec:combined-uncertainty-worst-case}.
\end{proof}

\begin{figure}
	\centering
	\includegraphics[width=\textwidth]{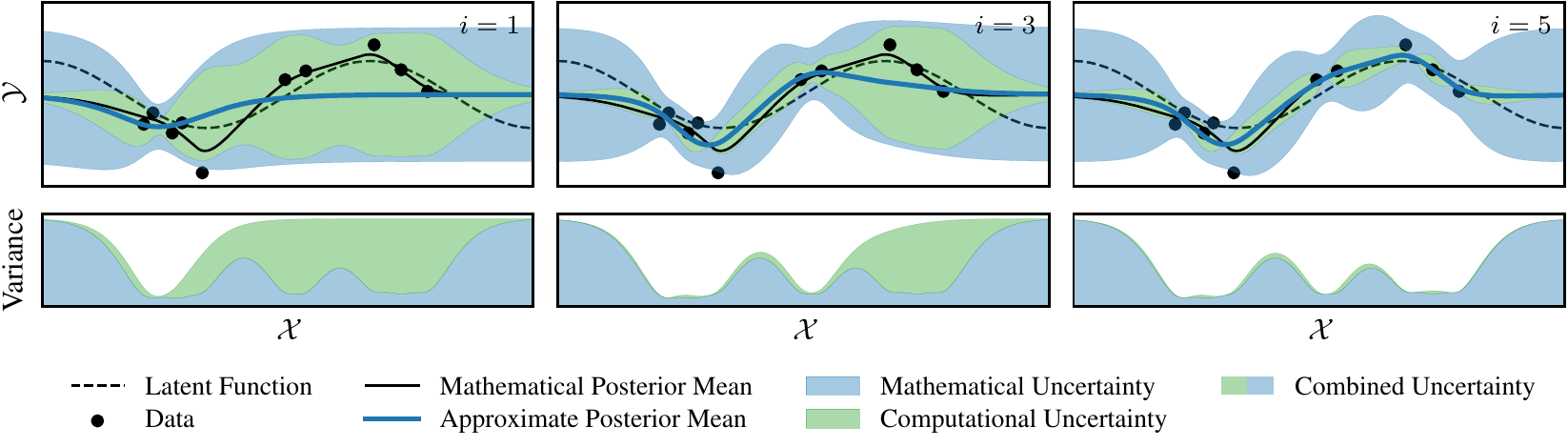}
	\caption[Caption]{\emph{Computational and combined uncertainty of IterGP as worst-case bounds.}\protect\footnotemark}
	\label{fig:covariance-worst-case-interpretation}
\end{figure}

\footnotetext{The combined (co-)variance decomposes into mathematical and computational
	covariances, as opposed to the combined standard deviation since \(\sqrt{\colordot{mplgreen} + \colordot{mplblue}} \neq
	\sqrt{\colordot{mplgreen}} + \sqrt{\colordot{mplblue}}\). The bottom panel thus illustrates the variance
	decomposition. However, to better illustrate \Cref{thm:worst-case-error}, in the upper panel we plot the combined standard deviation \(\sqrt{\colordot{mplgreen}+\colordot{mplblue}}\) and
	computational standard deviation \(\sqrt{\colordot{mplgreen}}\) within it, in line with standard GP plotting practice.}


\Cref{thm:worst-case-error} rigorously explains why the combined (mathematical + computational)
uncertainty \(k_\idxiter\) is the
correct object characterizing our belief about the latent function \(h\), given that we are in the
limited compute regime. In the same way that the mathematical uncertainty is a tight bound on the
distance to all functions \(g\) which could have produced the data (see \cite[Prop.~3.8]{Kanagawa2018}),
the combined uncertainty is a tight bound on all functions \(g\) which would have produced the same
computations. 

\subsection{Pointwise Convergence of the Posterior Mean}

In particular, as \Cref{cor:worst-case-error} shows and \Cref{fig:covariance-worst-case-interpretation} illustrates, the computational uncertainty (\colordot{mplgreen}) is a pointwise bound on the
relative distance to the mathematical posterior mean \eqref{eqn:pointwise-convergence-to-posterior-mean} and \emph{the combined uncertainty (\colordot{mplgreen} + \colordot{mplblue}) is a pointwise bound
	on the relative distance to the true latent function} \eqref{eqn:pointwise-convergence-to-latent-function}.

\begin{restatable}[Pointwise Convergence of the Posterior Mean]{corollary}{corworstcaseerror}
	\label{cor:worst-case-error}
	Assume the conditions of \Cref{thm:worst-case-error} hold and assume the latent function \(h \in
	\mathcal{H}_{k^\sigma}\). Let \(\mu_*\) be the corresponding mathematical posterior mean and \(\mu_\idxiter\) the posterior mean computed by \Cref{alg:itergp}. Then it holds that
	\begin{align}
		\frac{\lvert h(\vx) - \mu_\idxiter(\vx) \rvert}{\hspace{1.2em} \norm{h}_{\mathcal{H}_{k^\sigma}}}  & \leq  \sqrt{k_\idxiter(\vx, \vx) + \sigma^2}, \quad \text{and} \label{eqn:pointwise-convergence-to-latent-function}                     \\
		\frac{\lvert \mu_*(\vx) - \mu_\idxiter(\vx) \rvert}{\hspace{1.3em} \norm{h}_{\mathcal{H}_{k^\sigma}}} & \leq  \sqrt{k_\idxiter^{\textup{comp}}(\vx, \vx)}.
		\label{eqn:pointwise-convergence-to-posterior-mean}
	\end{align}
\end{restatable}
\begin{proof}
	This follows immediately from \Cref{thm:worst-case-error} by recognizing that \(h /
	\norm{h}_{\mathcal{H}_{k^\sigma}}\) has unit norm.
\end{proof}

It is worth noting that \Cref{thm:worst-case-error} and \Cref{cor:worst-case-error} generally \emph{do not hold for other GP approximations}. They explicitly rely on \(\invapprox_\idxiter \shat{\kernmat}\) being the \(\shat{\kernmat}\)-orthogonal projection onto the space spanned by the actions (see \Cref{lem:geometric-properties}). Since orthogonal projections are unique, if another GP approximation is such a projection and therefore satisfies \Cref{thm:worst-case-error}, it is in fact an instance of IterGP.

\section{Experiments}
To demonstrate the effects of quantifying computational uncertainty we perform GP regression on
synthetic and
benchmark datasets for the two most common GP approximations in the large-scale setting, SVGP
\cite{Hensman2013} and CGGP \cite{Gardner2018}, and their direct analogs from our class
of methods. An implementation of \Cref{alg:itergp}, based on KeOps \cite{Charlier2021} and
ProbNum \cite{Wenger2021}, is available at:

\centerline{\url{https://github.com/JonathanWenger/itergp}}

\textbf{Experimental Setup}
We consider a synthetic dataset of iid uniformly sampled inputs \(\vx_j \in [-1, 1]^d\)
with \(y(\vx) = \sin(\pi \vx^\top \vone) + \varepsilon\), where \(\varepsilon \sim
\Normal(0, \sigma^2)\),
as well as a range of UCI datasets \cite{Dua2017} with training set sizes \(n=5,287\) to
\(57,247\), dimensions \(d =9\) to \(26\) and standardized features.
All experiments were run on an NVIDIA GeForce RTX 2080 Ti graphics card.
We perform GP regression using a zero mean prior and a Mat\' ern(\(\frac{1}{2}\))
kernel (for other kernels see \Cref{sec:additional-experimental-results}).
All experiments were run \(10\) times with randomly sampled training and test splits of \(90/10\)
and we report average metrics with 95\% confidence intervals.

\begin{figure}[b!]
	\begin{subfigure}[b]{0.8\textwidth}
		\centering
		\includegraphics[width=\textwidth]{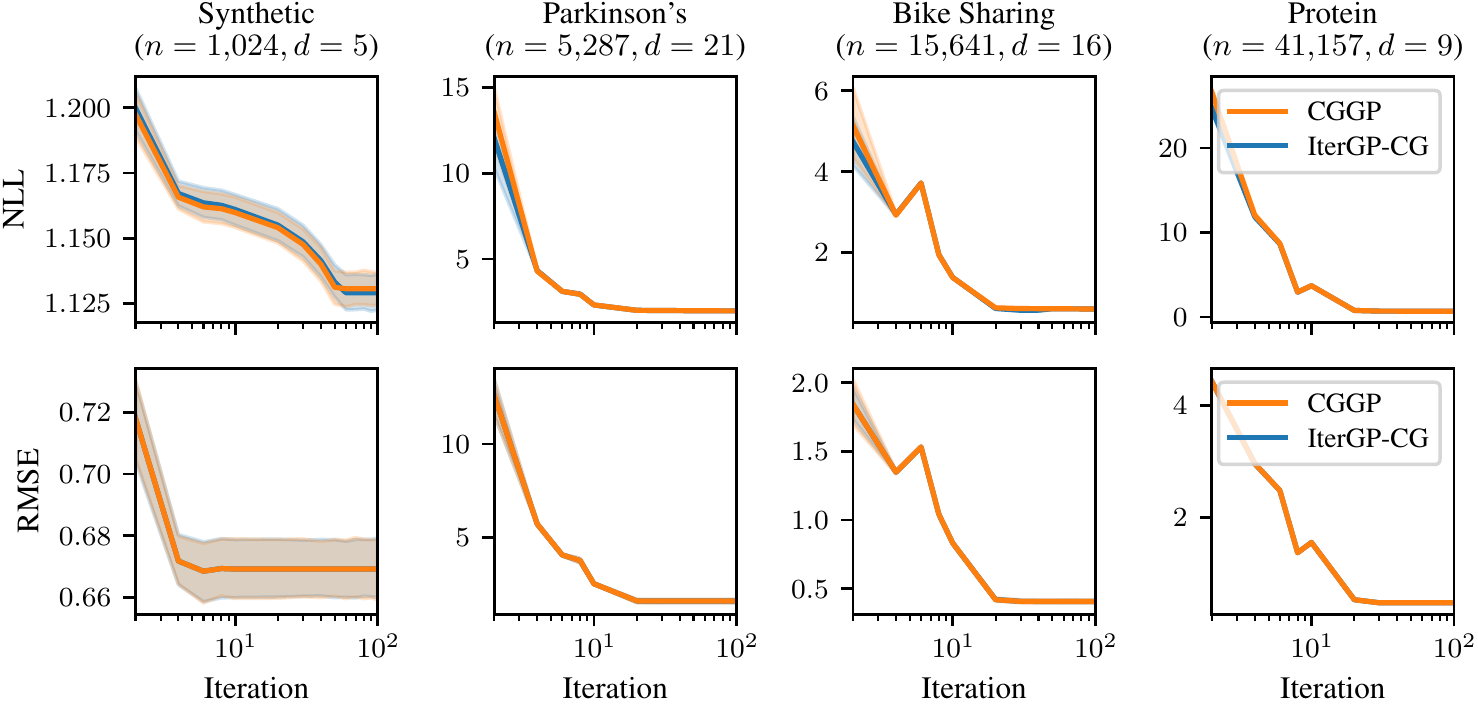}
		\caption{Generalization on synthetic and UCI benchmark datasets.}
		\label{fig:exp-itergp-cggp-generalization}
	\end{subfigure}\hfill
	\begin{subfigure}[b]{0.19\textwidth}
		\centering
		\includegraphics[width=\textwidth]{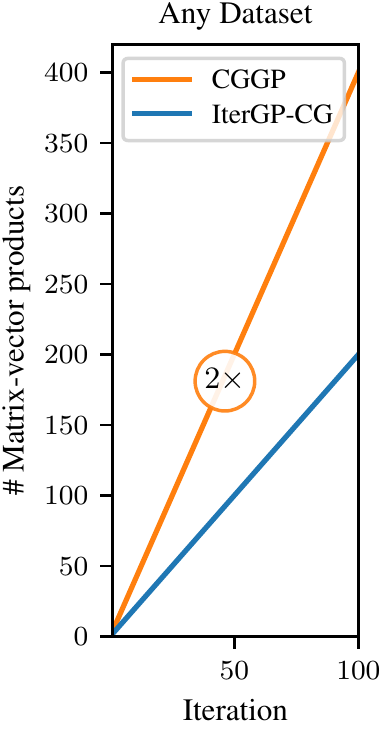}
		\caption{Comp. Cost}
		\label{fig:exp-itergp-cggp-cost}
	\end{subfigure}
	\caption{\emph{Generalization of CGGP and its closest IterGP analog.} (a) GP regression using a Mat\'ern(\(\frac{1}{2}\)) kernel on UCI
		datasets. The plot shows the average generalization error in terms of NLL and RMSE for an
		increasing number of solver iterations. The posterior mean of IterGP-CG and CGGP is
		identical, which explains the identical RMSE. However, CGGP performs additional computation for the posterior covariance as (b)
		illustrates, which is not needed since IterGP-CG has identical NLL.}
	\label{fig:exp-itergp-cggp}%
\end{figure}

\textbf{IterGP reduces the necessary computations for CG-based GP inference.}
We compare IterGP to the CG-based GP inference used in the GPyTorch library
\citep{Gardner2018}.
For all datasets, we select hyperparameters using the training procedure of
\citet{Wenger2022}.
As we show in \Cref{thm:itergp-cg},
the posterior mean of IterGP with (conjugate) residual actions is exactly equivalent to performing
CG to compute the representer weights.
Therefore, both methods produce the exact same posterior mean estimate and thus achieve the same
RMSE as a function of CG iterations (\Cref{fig:exp-itergp-cggp}, bottom).
The primary difference between the two methods is in the posterior variance.
The combined variance estimate of IterGP is essentially ``free''
in the sense that it reuses terms from the posterior mean calculation.
In contrast, computing the posterior variance with CG requires \(n_\diamond\) additional linear solves
(\(\shat\kernmat^{-1} \vx_{\diamond 1}, \ldots, \shat\kernmat^{-1}
\vx_{\diamond n_\diamond}\)).
GPyTorch relies on the Lanczos Variance Estimate technique \cite{Pleiss2018} which
essentially warm-starts each of these solves by reusing quantities from the linear solve
\(\shat{\kernmat}^{-1}k(\traindata, \vx_{\diamond 1})\).
While this approach produces reliable variance estimates that converge to the true posterior
variance,
it requires additional computation: at least one set of additional CG iterations to compute
\(\shat{\kernmat}^{-1}k(\traindata, \vx_{\diamond 1})\).
In \Cref{fig:exp-itergp-cggp-generalization} (top), we see that IterGP and GPyTorch's CGGP achieve nearly identical NLL,
suggesting that both methods produce variances that yield similar generalization.
The key difference between the methods is that 1) unlike CGGP, IterGP's variances exactly capture
both mathematical and computational uncertainty,
and 2) IterGP's variances require no additional solves, resulting in half as much computation as
GPyTorch's CGGP implementation (see \Cref{fig:exp-itergp-cggp-cost}).

\textbf{Quantifying computational uncertainty improves generalization of inducing point methods.}
To understand the benefits of quantifying computational uncertainty,
we compare the linear-time SVGP method (which does not quantify computational uncertainty)
with the closest (quadratic-time) inducing point analog from our proposed IterGP framework (see
\Cref{sec:connection-gp-approx}).
While the IterGP method is inherently more expensive than SVGP,
our goal is simply to demonstrate that inducing points can yield far more accuracy if one has the
budget to account for computational uncertainty.
To that end, we compare SVGP against IterGP using the same set of randomly-placed inducing points.
We identify a set of kernel hyperparameters by optimizing the ELBO of SVGP on the training data,
using these for both SVGP and IterGP.
As \Cref{fig:exp-itergp-svgp} shows, we find that across all datasets that
IterGP offers better RMSE and NLL than SVGP,
despite the fact that the hyperparameters are chosen to favor SVGP.
This suggests that the extra computation needed to quantify computational uncertainty
can more ``effectively'' utilize a set of inducing points for predictive models.

\begin{figure}[h!]
	\centering
	\includegraphics[width=\textwidth]{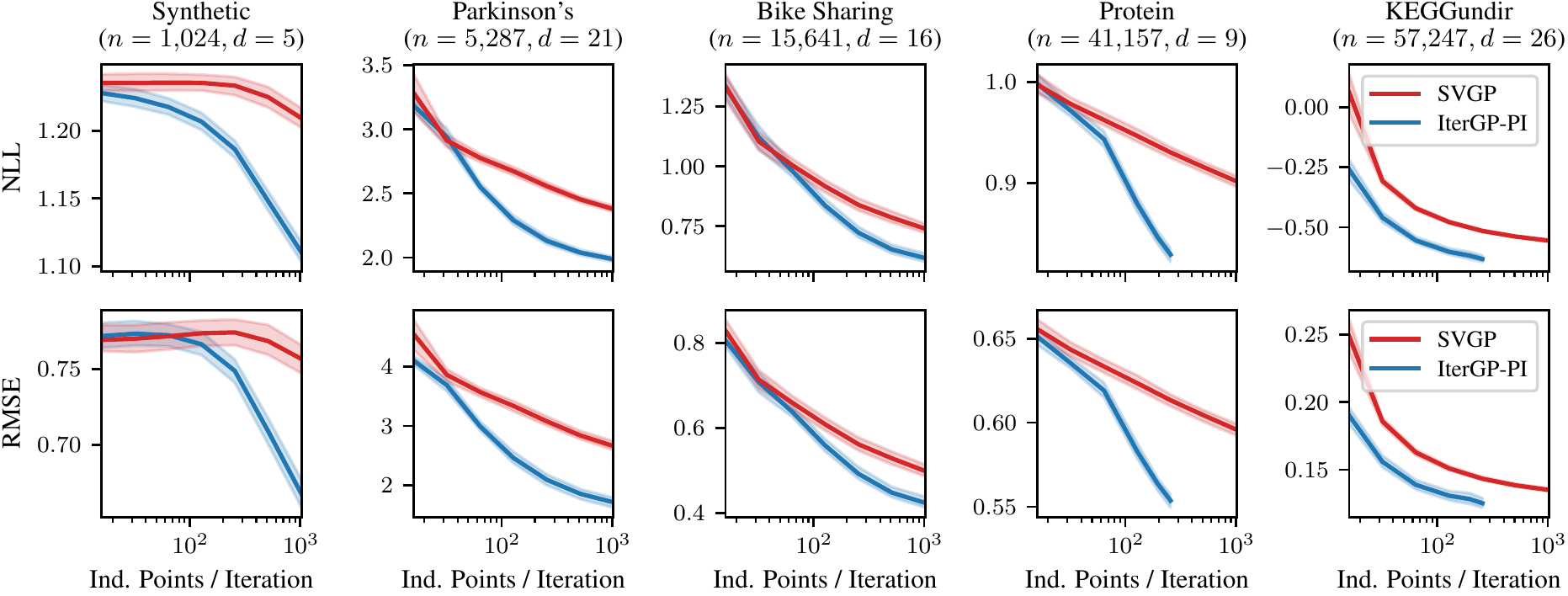}
	\caption{\emph{Generalization of SVGP and its closest IterGP analog.} GP regression using a Mat\'ern(\(\frac{1}{2}\)) kernel on UCI
		datasets. The plot shows the average generalization error in terms of NLL and RMSE for an
		increasing number of identical inducing points. After a small number of inducing points relative to
		the size of the training data, IterGP has significantly lower generalization error than SVGP.}
	\label{fig:exp-itergp-svgp}
\end{figure}

\section{Conclusion}

Scalable GP approximations inevitably introduce error, leading to a worse model for the latent
function in question.
This work demonstrates that it is possible to account for \emph{both} uncertainty
arising from limited data \emph{and} uncertainty arising from limited computation
\emph{exactly} -- which as we show improves model performance. IterGP methods return this
combined uncertainty which crucially represents a dataset-specific, pointwise worst-case bound on
the error to the true latent function. At its core, IterGP performs repeated matrix-vector
multiplication resulting in quadratic complexity. Since modern computing architectures (i.e. GPUs)
have been specifically designed for this operation at scale, iterative approaches for GP
approximation are becoming competitive with theoretically cheaper approximations, like inducing point
methods \cite{Gardner2018,Wang2019}. Finally, in addition to the general utility of IterGP, we expect
this class of methods to be particularly useful in applications where accurate uncertainty
quantification is important or, due to its inherently online nature, where data is acquired
sequentially such as in active learning and Bayesian optimization.

\newpage

\begin{ack}
	JW, MP and PH gratefully acknowledge financial support by the European Research Council through ERC StG Action 757275 / PANAMA; the DFG Cluster of Excellence ``Machine Learning - New Perspectives for Science'', EXC 2064/1, project number 390727645; the German Federal Ministry of Education and Research (BMBF) through the Tübingen AI Center (FKZ: 01IS18039A); and funds from the Ministry of Science, Research and Arts of the State of Baden-Württemberg. The authors thank the International Max Planck Research School for Intelligent Systems (IMPRS-IS) for supporting JW and MP. JPC and GP are supported by the Simons Foundation, the McKnight Foundation,
	the Grossman Center, and the Gatsby
	Charitable Trust. The authors would like to thank Hanna Dettki, Frank Schneider, Lukas Tatzel and all anonymous reviewers for helpful feedback on an earlier version of this manuscript.
\end{ack}

{\small
\bibliographystyle{unsrtnat}
\bibliography{references}
}

\section*{Checklist}

\begin{enumerate}

	\item For all authors...
	      \begin{enumerate}
		      \item Do the main claims made in the abstract and introduction accurately reflect the paper's
		            contributions and scope?
		            \answerYes{}
		      \item Did you describe the limitations of your work?
		            \answerYes{}
		      \item Did you discuss any potential negative societal impacts of your work?
		            \answerNA{}
		      \item Have you read the ethics review guidelines and ensured that your paper conforms to them?
		            \answerYes{}
	      \end{enumerate}

	\item If you are including theoretical results...
	      \begin{enumerate}
		      \item Did you state the full set of assumptions of all theoretical results?
		            \answerYes{}
		      \item Did you include complete proofs of all theoretical results?
		            \answerYes{Proofs are included in the supplementary material.}
	      \end{enumerate}

	\item If you ran experiments...
	      \begin{enumerate}
		      \item Did you include the code, data, and instructions needed to reproduce the main experimental results
		            (either in the supplemental material or as a URL)?
		            \answerYes{}
		      \item Did you specify all the training details (e.g., data splits, hyperparameters, how they were
		            chosen)?
		            \answerYes{}
		      \item Did you report error bars (e.g., with respect to the random seed after running experiments multiple
		            times)?
		            \answerYes{}
		      \item Did you include the total amount of compute and the type of resources used (e.g., type of GPUs,
		            internal cluster, or cloud provider)?
		            \answerYes{}
	      \end{enumerate}

	\item If you are using existing assets (e.g., code, data, models) or curating/releasing new assets...
	      \begin{enumerate}
		      \item If your work uses existing assets, did you cite the creators?
		            \answerYes{}
		      \item Did you mention the license of the assets?
		            \answerNA{}
		      \item Did you include any new assets either in the supplemental material or as a URL?
		            \answerNA{}
		      \item Did you discuss whether and how consent was obtained from people whose data you're using/curating?
		            \answerNA{}
		      \item Did you discuss whether the data you are using/curating contains personally identifiable
		            information or offensive content?
		            \answerNA{}
	      \end{enumerate}

	\item If you used crowdsourcing or conducted research with human subjects...
	      \begin{enumerate}
		      \item Did you include the full text of instructions given to participants and screenshots, if applicable?
		            \answerNA{}
		      \item Did you describe any potential participant risks, with links to Institutional Review Board (IRB)
		            approvals, if applicable?
		            \answerNA{}
		      \item Did you include the estimated hourly wage paid to participants and the total amount spent on
		            participant compensation?
		            \answerNA{}
	      \end{enumerate}

\end{enumerate}

\clearpage


\makeatletter
\newcommand{\suptitle}{Supplementary Material:\\\papertitle}
\renewcommand{\@title}{\suptitle}
\newcommand{\thanks}[1]{\footnotemark[1]}
\renewcommand{\@author}{\authorinfo}

\par
\begingroup
\renewcommand{\thefootnote}{\fnsymbol{footnote}}
\renewcommand{\@makefnmark}{\hbox to \z@{$^{\@thefnmark}$\hss}}
\renewcommand{\@makefntext}[1]{%
	\parindent 1em\noindent
	\hbox to 1.8em{\hss $\m@th ^{\@thefnmark}$}#1
}
\@maketitle
\@thanks
\endgroup
\let\maketitle\relax
\let\thanks\relax
\makeatother

\beginsupplementary
\startcontents[supplementary]

This supplementary material contains additional results and in particular proofs
for all theoretical statements. References referring to sections, equations or theorem-type
environments within this
document are prefixed with `S', while references to, or results from, the main paper are stated as
is.

	{\small
		\vspace{1em}
		\printcontents[supplementary]{}{1}{}
	}

\section{Connections to Other GP Approximations}

\subsection{Pivoted Cholesky Decomposition}

\begin{figure}
	\begin{minipage}[t]{.48\textwidth}
		\begin{figure}[H]
			\centering
			\adjustbox{max width=\textwidth}{
				\tikzset{every picture/.style={line width=0.75pt}} 

\begin{tikzpicture}[x=0.75pt,y=0.75pt,yscale=-1,xscale=1]

\draw  [fill={rgb, 255:red, 205; green, 205; blue, 205 }  ,fill opacity=1 ][line width=0.75]  (70,0.5) -- (126.75,0.5) -- (126.75,57.25) -- (70,57.25) -- cycle ;
\draw  [line width=0.75]  (219.89,0.5) -- (276.64,0.5) -- (276.64,57.25) -- (219.89,57.25) -- cycle ;
\draw  [line width=0.75]  (155.65,0.5) -- (212.4,0.5) -- (212.4,57.25) -- (155.65,57.25) -- cycle ;
\draw  [fill={rgb, 255:red, 74; green, 144; blue, 226 }  ,fill opacity=1 ] (169.93,14.42) -- (169.93,57.25) -- (155.65,57.25) -- (155.65,0.5) -- cycle ;
\draw  [fill={rgb, 255:red, 74; green, 144; blue, 226 }  ,fill opacity=1 ] (234.17,14.42) -- (276.99,14.42) -- (276.64,0.5) -- (219.89,0.5) -- cycle ;
\draw  [fill={rgb, 255:red, 74; green, 144; blue, 226 }  ,fill opacity=1 ] (126.75,0.5) -- (127.1,14.42) -- (84.28,14.42) -- (84.28,57.25) -- (70,57.25) -- (70,0.5) -- cycle ;
\draw  [fill={rgb, 255:red, 205; green, 205; blue, 205 }  ,fill opacity=1 ][line width=0.75]  (70,71.52) -- (126.75,71.52) -- (126.75,128.27) -- (70,128.27) -- cycle ;
\draw  [line width=0.75]  (219.89,71.52) -- (276.64,71.52) -- (276.64,128.27) -- (219.89,128.27) -- cycle ;
\draw  [line width=0.75]  (155.65,71.52) -- (212.4,71.52) -- (212.4,128.27) -- (155.65,128.27) -- cycle ;
\draw  [fill={rgb, 255:red, 74; green, 144; blue, 226 }  ,fill opacity=1 ] (169.93,85.44) -- (169.93,128.27) -- (155.65,128.27) -- (155.65,71.52) -- cycle ;
\draw  [fill={rgb, 255:red, 74; green, 144; blue, 226 }  ,fill opacity=1 ] (234.17,85.44) -- (276.99,85.44) -- (276.64,71.52) -- (219.89,71.52) -- cycle ;
\draw  [fill={rgb, 255:red, 74; green, 144; blue, 226 }  ,fill opacity=1 ] (126.75,71.52) -- (127.1,85.44) -- (84.28,85.44) -- (84.28,128.27) -- (70,128.27) -- (70,71.52) -- cycle ;
\draw  [fill={rgb, 255:red, 18; green, 90; blue, 176 }  ,fill opacity=1 ] (184.2,99.36) -- (184.2,128.62) -- (169.93,128.27) -- (169.93,85.44) -- cycle ;
\draw  [fill={rgb, 255:red, 18; green, 90; blue, 176 }  ,fill opacity=1 ] (248.44,99.71) -- (276.99,100.07) -- (276.99,85.44) -- (234.17,85.44) -- cycle ;
\draw  [fill={rgb, 255:red, 18; green, 90; blue, 176 }  ,fill opacity=1 ] (127.1,85.44) -- (127.1,100.07) -- (98.55,99.71) -- (98.55,128.62) -- (84.28,128.27) -- (84.28,85.44) -- cycle ;
\draw  [dash pattern={on 0.84pt off 2.51pt}]  (70,0.5) -- (126.75,57.25) ;
\draw  [dash pattern={on 0.84pt off 2.51pt}]  (70,71.52) -- (126.75,128.27) ;
\draw  [dash pattern={on 0.84pt off 2.51pt}]  (155.65,0.5) -- (212.4,57.25) ;
\draw  [dash pattern={on 0.84pt off 2.51pt}]  (219.89,0.5) -- (276.64,57.25) ;
\draw  [dash pattern={on 0.84pt off 2.51pt}]  (155.83,70.99) -- (212.58,127.73) ;
\draw  [dash pattern={on 0.84pt off 2.51pt}]  (219.89,71.52) -- (276.64,128.27) ;
\draw  [fill={rgb, 255:red, 205; green, 205; blue, 205 }  ,fill opacity=1 ][line width=0.75]  (70,142.9) -- (126.75,142.9) -- (126.75,199.64) -- (70,199.64) -- cycle ;
\draw  [line width=0.75]  (219.89,142.9) -- (276.64,142.9) -- (276.64,199.64) -- (219.89,199.64) -- cycle ;
\draw  [line width=0.75]  (155.65,142.9) -- (212.4,142.9) -- (212.4,199.64) -- (155.65,199.64) -- cycle ;
\draw  [fill={rgb, 255:red, 74; green, 144; blue, 226 }  ,fill opacity=1 ] (169.93,156.82) -- (169.93,199.64) -- (155.65,199.64) -- (155.65,142.9) -- cycle ;
\draw  [fill={rgb, 255:red, 74; green, 144; blue, 226 }  ,fill opacity=1 ] (234.17,156.82) -- (276.99,156.82) -- (276.64,142.9) -- (219.89,142.9) -- cycle ;
\draw  [fill={rgb, 255:red, 74; green, 144; blue, 226 }  ,fill opacity=1 ] (126.75,142.9) -- (127.1,156.82) -- (84.28,156.82) -- (84.28,199.64) -- (70,199.64) -- (70,142.9) -- cycle ;
\draw  [fill={rgb, 255:red, 18; green, 90; blue, 176 }  ,fill opacity=1 ] (184.2,170.74) -- (184.2,200) -- (169.93,199.64) -- (169.93,156.82) -- cycle ;
\draw  [fill={rgb, 255:red, 18; green, 90; blue, 176 }  ,fill opacity=1 ] (248.44,171.09) -- (276.99,171.45) -- (276.99,156.82) -- (234.17,156.82) -- cycle ;
\draw  [fill={rgb, 255:red, 18; green, 90; blue, 176 }  ,fill opacity=1 ] (127.1,156.82) -- (127.1,171.45) -- (98.55,171.09) -- (98.55,200) -- (84.28,199.64) -- (84.28,156.82) -- cycle ;
\draw  [dash pattern={on 0.84pt off 2.51pt}]  (155.83,142.36) -- (212.58,199.11) ;
\draw  [dash pattern={on 0.84pt off 2.51pt}]  (219.89,142.9) -- (276.64,199.64) ;
\draw  [fill={rgb, 255:red, 11; green, 50; blue, 96 }  ,fill opacity=1 ] (127.1,171.45) -- (127.1,185.72) -- (112.83,185.72) -- (112.83,200) -- (98.55,200) -- (98.55,171.81) -- cycle ;
\draw  [dash pattern={on 0.84pt off 2.51pt}]  (70,142.9) -- (126.75,199.64) ;
\draw  [fill={rgb, 255:red, 11; green, 50; blue, 96 }  ,fill opacity=1 ] (198.3,185.19) -- (198.48,200) -- (184.2,200) -- (184.03,171.27) -- cycle ;
\draw  [fill={rgb, 255:red, 11; green, 50; blue, 96 }  ,fill opacity=1 ] (262.72,185.37) -- (276.99,185.72) -- (276.99,171.45) -- (248.44,171.45) -- cycle ;

\draw (8.5,21.27) node [anchor=north west][inner sep=0.75pt]    {$\idxiter = 1$};
\draw (134,21.27) node [anchor=north west][inner sep=0.75pt]    {$\approx $};
\draw (8.5,92.29) node [anchor=north west][inner sep=0.75pt]    {$\idxiter = 2$};
\draw (134,92.29) node [anchor=north west][inner sep=0.75pt]    {$\approx $};
\draw (8.5,163.67) node [anchor=north west][inner sep=0.75pt]    {$\idxiter = 3$};
\draw (134,163.67) node [anchor=north west][inner sep=0.75pt]    {$\approx $};
\draw (75,209.4) node [anchor=north west][inner sep=0.75pt]    {$\textcolor{gray}{\mP^\top}\shat{\kernmat}\textcolor{gray}{\mP}$};
\draw (175,208.4) node [anchor=north west][inner sep=0.75pt]    {$\mL_{\idxiter}$};
\draw (238,207.4) node [anchor=north west][inner sep=0.75pt]    {$\mL_{\idxiter}^{\top}$};

\end{tikzpicture}
			}
		\end{figure}

	\end{minipage}
	\begin{minipage}[t]{.52\textwidth}
		\begin{algorithm}[H]
			\caption{(Pivoted) Cholesky Decomposition\label{alg:cholesky}}
			\textbf{Input:} kernel matrix \(\shat{\kernmat}\), \textcolor{gray}{permutation matrix \(\mP\)}\\
\textbf{Output:} lower triangular \(\cholfac_\idxiter\), s.t. \(\cholfac_\idxiter \cholfac_\idxiter^\top \approx \textcolor{gray}{\mP^\top}\shat{\kernmat}\textcolor{gray}{\mP}\)
\begin{algorithmic}[1]
    \Procedure{\textsc{Cholesky}}{$\shat{\kernmat}, \textcolor{gray}{\mP}$}
    \State \(\mA \gets \textcolor{gray}{\mP^\top}\shat{\kernmat}\textcolor{gray}{\mP}\)
    \For{\(\idxiter \in \{1, \dots, n\}\)}
    \State \(\vl_{\idxiter} \gets \mA_{:\idxiter}/\sqrt{\mA_{\idxiter\idxiter}}\)
    \State \(\mA \gets \mA - \vl_\idxiter\vl_\idxiter^\top \textcolor{gray}{ = \mP^\top \shat{\kernmat} \mP - \cholfac_\idxiter \cholfac_\idxiter^\top}\)
    \State \(\cholfac_{\idxiter} = \begin{pmatrix} \cholfac_{\idxiter-1} & \vl_{\idxiter} \end{pmatrix}\)
    \EndFor
    \State \Return \(\cholfac_\idxiter\)
    \EndProcedure
\end{algorithmic}

		\end{algorithm}
	\end{minipage}
	\caption{\emph{Cholesky decomposition.} Every column added to the lower triangular Cholesky factor \(\mL\) defines
		the \(i\)th ``right angle ruler''-pattern in \(\textcolor{gray}{\mP^\top}\shat{\kernmat}\textcolor{gray}{\mP}\). The bottom right matrix
		in gray given by
		\(\textcolor{gray}{\mP^\top}\shat{\kernmat}\textcolor{gray}{\mP} - \cholfac_\idxiter \cholfac_\idxiter^\top = \textcolor{gray}{\mP^\top}\shat{\kernmat}\textcolor{gray}{\mP} - \sum_{j=1}^\idxiter
		\vl_j\vl_j^\top\)
		changes every
		iteration.}
	\label{fig:cholesky-illustration}
\end{figure}

\begin{restatable}[Pivoted Cholesky Decomposition]{theorem}{thmcholesky}
	\label{thm:itergp-cholesky}
	Let \((j_i)_{i=1}^n\) be a set of indices defining the pivot elements of the pivoted Cholesky decomposition and \(\mP \in \Rnn\) the corresponding permutation matrix. Assume the actions of \Cref{alg:itergp} are given by the standard unit vectors \(\action_\idxiter = \mP \ve_\idxiter = \ve_{j_\idxiter}\), i.e. 
	\begin{equation}
		(\action_\idxiter)_j = (\ve_{j_\idxiter})_j = \begin{cases} 1 &\text{if } j = j_i\\ 0 & \text{otherwise}\end{cases}.
	\end{equation}
	Then \Cref{alg:itergp}
	recovers the pivoted Cholesky decomposition, i.e. it holds for all \(\idxiter \in \{0, \dots,
	n\}\) that
	\begin{equation}
		\label{eqn:proof-cholesky-induction-hypothesis}
		\mP^\top \kernmatapprox_\idxiter \mP = \cholfac_\idxiter \cholfac_\idxiter^\top,
	\end{equation}
	where \(\cholfac_\idxiter \in \R^{n \times \idxiter}\) is the (partial) Cholesky factor of \(\mP^\top \shat{\mK} \mP\) as computed by \Cref{alg:cholesky}.
\end{restatable}
\begin{proof}
	We proceed by induction. Assume \eqref{eqn:proof-cholesky-induction-hypothesis} holds after \(\idxiter\) iterations of
	\Cref{alg:itergp}.
	For the base case \(\idxiter = 0\), it holds by assumption that \(\mP^\top\kernmatapprox_0\mP = \mP^\top \shat{\kernmat}
	\invapprox_0 \shat{\kernmat}\mP = \mZero\). Now for the induction step \(\idxiter \to \idxiter + 1\), we have
	\begin{align*}
		\frac{1}{\searchdirsqnorm_{\idxiter + 1}} \shat{\kernmat} \searchdir_\idxiter \searchdir_\idxiter^\top \shat{\kernmat}                                                                                                                                                                                                                  
		                                                                               & = \frac{1}{\searchdirsqnorm_{\idxiter + 1}} \shat{\kernmat} \qoicov_{\idxiter} \shat{\kernmat} \action_{\idxiter + 1} \action_{\idxiter + 1}^\top \shat{\kernmat} \qoicov_{\idxiter} \shat{\kernmat}                                                                                                                                                                                                                                                                                   \\
		                                                                               & =\frac{1}{\searchdirsqnorm_{\idxiter + 1}} \shat{\kernmat} (\qoicov_0 - \invapprox_\idxiter)\shat{\kernmat} \action_{\idxiter + 1} \action_{\idxiter + 1}^\top \shat{\kernmat} (\qoicov_0 - \invapprox_\idxiter) \shat{\kernmat}                                                                                                                                                                                                                                                       \\
		                                                                               & =\frac{1}{\searchdirsqnorm_{\idxiter + 1}} (\shat{\kernmat} - \kernmatapprox_\idxiter) \action_{\idxiter + 1} \action_{\idxiter + 1}^\top  (\shat{\kernmat} - \kernmatapprox_\idxiter)                                                                                                                                                                                                                                                                                   \\
		                                                                               & \overset{\mathrm{IH}}{=} \frac{1}{\searchdirsqnorm_{\idxiter + 1}} (\shat{\kernmat} - \mP\cholfac_\idxiter \cholfac_\idxiter^\top\mP^\top) \action_{\idxiter + 1} \action_{\idxiter + 1}^\top  (\shat{\kernmat} - \mP\cholfac_\idxiter \cholfac_\idxiter^\top\mP^\top)                                                                                                                                                                  \\
		                                                                               & = \frac{(\shat{\kernmat} - \mP\cholfac_\idxiter \cholfac_\idxiter^\top\mP^\top) \mP\ve_{\idxiter + 1}}{\sqrt{\ve_{\idxiter + 1}^\top\mP^\top(\shat{\kernmat} - \mP\cholfac_\idxiter \cholfac_\idxiter^\top\mP^\top) \mP\ve_{\idxiter + 1}}} \frac{\ve_{\idxiter + 1}^\top\mP^\top (\shat{\kernmat} - \mP\cholfac_\idxiter \cholfac_\idxiter^\top\mP^\top)}{\sqrt{\ve_{\idxiter + 1}^\top\mP^\top(\shat{\kernmat} - \mP\cholfac_\idxiter \cholfac_\idxiter^\top\mP^\top) \mP\ve_{\idxiter + 1}}}                                                                               \\
																					   &= \frac{\mP(\mP^\top\shat{\kernmat}\mP - \cholfac_\idxiter \cholfac_\idxiter^\top) \ve_{\idxiter + 1}}{\sqrt{\ve_{\idxiter + 1}^\top(\mP^\top\shat{\kernmat}\mP - \cholfac_\idxiter \cholfac_\idxiter^\top) \ve_{\idxiter + 1}}} \frac{\ve_{\idxiter + 1}^\top(\mP^\top\shat{\kernmat}\mP - \cholfac_\idxiter \cholfac_\idxiter^\top)\mP^\top}{\sqrt{\ve_{\idxiter + 1}^\top(\mP^\top\shat{\kernmat}\mP - \cholfac_\idxiter \cholfac_\idxiter^\top) \ve_{\idxiter + 1}}}  \\
		                                                                               & =  \mP\vl_{\idxiter+1} \vl_{\idxiter+1}^\top\mP^\top.
	\end{align*}
	where \(\vl_{\idxiter+1}\) is given by \Cref{alg:cholesky}. Combining this with one more use of the induction hypothesis we obtain
	\begin{align*}
		\mP^\top \kernmatapprox_{\idxiter + 1}\mP &= \mP^\top\kernmatapprox_{\idxiter}\mP + \frac{1}{\searchdirsqnorm_{\idxiter + 1}} \mP^\top\shat{\kernmat} \searchdir_{\idxiter+1}
		\searchdir_{\idxiter+1}^\top \shat{\kernmat}\mP\\
		&= \cholfac_\idxiter \cholfac_\idxiter^\top + \vl_{\idxiter+1} \vl_{\idxiter+1}^\top =
		\begin{pmatrix}\cholfac_\idxiter & \vl_{\idxiter +1} \end{pmatrix} \begin{pmatrix}\cholfac_\idxiter^\top\\ \vl_{\idxiter +1}^\top \end{pmatrix} = \cholfac_{\idxiter+1} \cholfac_{\idxiter+1}^\top
	\end{align*}
	This proves the claim.

\end{proof}

\subsection{Singular / Eigenvalue Decomposition}

\begin{restatable}[Singular / Eigenvalue Decomposition]{theorem}{thmsvd}
	\label{thm:itergp-svd}
	Let the actions \(\action_\idxiter =
	\vu_\idxiter\) of \Cref{alg:itergp} be given by the eigenvectors \(\vu_\idxiter\) of \(\shat{\kernmat}\) in arbitrary
	order. Then for \(\idxiter \in \{1,   \dots,  n\}\) it holds that
	\begin{align*}
		\invapprox_\idxiter     & = \mU_\idxiter \mLambda_\idxiter^{-1} \mU_\idxiter^\top = \textsc{SVD}_\idxiter(\shat{\kernmat}^{-1}) \\
		\kernmatapprox_\idxiter & =  \mU_\idxiter \mLambda_\idxiter \mU_\idxiter^\top = \textsc{SVD}_\idxiter(\shat{\kernmat}),
	\end{align*}
	where \(\mU = \begin{pmatrix} \vu_1, \dots, \vu_\idxiter \end{pmatrix} \in \R^{n \times \idxiter}\) and \(\mLambda =
	\operatorname{diag}(\lambda_1, \dots, \lambda_\idxiter) \in \R^{\idxiter \times \idxiter}\) is the diagonal matrix of eigenvalues of
	\(\shat{\kernmat}\) with the order given by the order of the actions.
\end{restatable}

\begin{proof}
	It holds by assumption and \cref{eqn:invapprox-batch-form}, that
	\begin{align*}
		\invapprox_\idxiter = \mActions_\idxiter(\mActions_\idxiter^\top \shat{\kernmat} \mActions_\idxiter)^{-1} \mActions_\idxiter^\top = \mU_\idxiter (\mU_\idxiter^\top \shat{\kernmat} \mU_\idxiter)^{-1} \mU_\idxiter^\top = \mU_\idxiter \mLambda_\idxiter^{-1} \mU_\idxiter^\top,
		\intertext{as well as}
		\kernmatapprox_\idxiter = \shat{\kernmat} \invapprox_\idxiter \shat{\kernmat} = \shat{\kernmat} \mU_\idxiter \mLambda_\idxiter^{-1} \mU_\idxiter^\top \shat{\kernmat} = \mU_\idxiter \mLambda_\idxiter\mLambda_\idxiter^{-1}\mLambda_\idxiter\mU_\idxiter^\top = \mU_\idxiter \mLambda_\idxiter \mU_\idxiter^\top
	\end{align*}
	This proves the claim.
\end{proof}

\subsection{Conjugate Gradient Method}


\begin{algorithm}[H]
	\caption{Preconditioned Conjugate Gradient Method \cite{Hestenes1952}\label{alg:cg}}
	\textbf{Input:} kernel matrix \(\shat{\kernmat}\), labels \(\labels\), prior mean \(\vmu\), preconditioner \(\shat{\precond}\)\\
\textbf{Output:} representer weights \(\qoimean_\idxiter \approx \shat{\kernmat}^{-1} (\labels - \vmu)\)
\begin{algorithmic}[1]
    \Procedure{\textsc{CG}}{$\shat{\kernmat}, \labels - \vmu, \shat{\precond}$}
    \State \(\qoimean_0 \gets \vzero\)
    \State \(\action_0 \gets \vzero\)
    \While{\(\norm{\vr_\idxiter}_2 > \max(\delta_{\textup{rtol}} \norm{\labels}_2, \delta_{\textup{atol}})\) \textbf{and} \(\idxiter < \idxiter_{\max}\)}
    \State \(\residual_{\idxiter-1} \gets (\labels - \vmu) - \shat{\kernmat} \qoimean_{\idxiter-1}\)
    \State \(\action_{\idxiter} \gets \shat{\precond}^{-1} \residual_{\idxiter-1} -
    \frac{(\shat{\precond}^{-1}\residual_{\idxiter-1})^\top \shat{\kernmat} \action_{\idxiter-1}}{\action_{\idxiter-1}^\top\shat{\kernmat} \action_{\idxiter-1}} \action_{\idxiter-1}\)
    \State \(\qoimean_\idxiter \gets \qoimean_{\idxiter-1} + \frac{(\shat{\precond}^{-1}\residual_{\idxiter-1})^\top \residual_{\idxiter-1}}{\action_{\idxiter}^\top\shat{\kernmat} \action_{\idxiter}} \action_{\idxiter}\)
    \EndWhile
    \State \Return \(\qoimean_\idxiter\)
    \EndProcedure
\end{algorithmic}

\end{algorithm}

\begin{restatable}[Preconditioned Conjugate Gradient Method]{theorem}{thmconjugategradients}
	\label{thm:itergp-cg}
	Let \(\shat{\precond} \in \Rnn\) be a symmetric positive definite preconditioner. Assume the actions of \Cref{alg:itergp} are given by
	\begin{equation}
		\label{eqn:cg-actions}
		\begin{aligned}
			\action_1^\txtCG          & = \shat{\precond}^{-1}\residual_0              \\
			\action_{\idxiter}^\txtCG & =\shat{\precond}^{-1} \residual_{\idxiter-1} -
			\frac{(\shat{\precond}^{-1}\residual_{\idxiter-1})^\top \shat{\kernmat} \action_{\idxiter-1}}{\action_{\idxiter-1}^\top\shat{\kernmat} \action_{\idxiter-1}} \action_{\idxiter-1} 
		\end{aligned}
	\end{equation}
	the preconditioned conjugate gradient method. Then \Cref{alg:itergp}
	recovers preconditioned CG initialized at \(\qoimean_0^{\txtCG} = \vzero\), i.e. it holds for
	\(\idxiter \in \{1,   \dots,  n\}\) that
	\begin{align}
		\action_\idxiter       & = \searchdir_\idxiter = \action_\idxiter^\txtCG \label{eqn:action-equivalence-cg} \\
		\qoimean_\idxiter      & = \qoimean_\idxiter^{\txtCG} \label{eqn:qoi-equivalence-cg}                       \\
		\residual_{\idxiter-1} & = \residual_{\idxiter-1}^\txtCG. \label{eqn:residual-equivalence-cg}
	\end{align}
\end{restatable}

\begin{proof}
	First note that by assumption \(\action_\idxiter = \action_\idxiter^\txtCG\) for all \(\idxiter\).
	We prove the remaining claims by induction. For the base case we have by assumption \(\searchdir_0
	= \qoicov_0\shat{\kernmat} \action_0 = \action_0 = \action_0^\txtCG\) and \(\qoimean_0 = \vzero =
	\qoimean_0^\txtCG\). Now for the induction step \(\idxiter \to \idxiter + 1\) assume the hypotheses
	\eqref{eqn:action-equivalence-cg}, \eqref{eqn:qoi-equivalence-cg} and \eqref{eqn:residual-equivalence-cg} hold \(\forall j \leq
	\idxiter\). Using the properties of CG it holds for \(j' < \idxiter\) that
	\begin{align}
		\action_\idxiter^\top \shat{\kernmat} \action_{j'}                                                                                                                                   & = 0\label{eqn:proof-cg-ih-conjugacy}                                                                                    \\
		(\shat{\precond}^{-1}\residual_\idxiter)^\top \action_{j'}                                                                                                                           & = 0\label{eqn:proof-cg-ih-resid-action-orth}                                                                            \\
		(\shat{\precond}^{-1}\residual_\idxiter)^\top \residual_{j'}                                                                                                                         & = 0\label{eqn:proof-cg-ih-resid-ortho}                                                                                  \\
		\langle \action_1, \dots, \action_\idxiter \rangle = \langle \residual_0, \shat{\precond}^{-1} \shat{\kernmat} \residual_0, \dots, (\shat{\precond}^{-1}\shat{\kernmat})^{\idxiter -1}\residual_0 \rangle & = \langle \shat{\precond}^{-1}\residual_0, \dots, \shat{\precond}^{-1}\residual_{\idxiter-1}\rangle\label{eqn:proof-cg-ih-span-equiv}
	\end{align}
	We now first show \(\shat{\kernmat}\)-conjugacy of the actions in iteration \(\idxiter +1\). We have for
	\(j \leq \idxiter\) that
	\begin{align*}
		\action_{\idxiter+1}^\top \shat{\kernmat} \action_{j} & = \big( \shat{\precond}^{-1} \residual_{\idxiter} -
		\frac{(\shat{\precond}^{-1}\residual_{\idxiter})^\top \shat{\kernmat} \action_{\idxiter}}{\action_{\idxiter}^\top\shat{\kernmat} \action_{\idxiter}} \action_{\idxiter}\big)^\top\shat{\kernmat} \action_j                                                                                                             \\
		                                               & = (\shat{\precond}^{-1}\residual_{\idxiter})^\top \shat{\kernmat} \action_j - \frac{(\shat{\precond}^{-1}\residual_{\idxiter})^\top \shat{\kernmat} \action_{\idxiter}}{\action_{\idxiter}^\top\shat{\kernmat} \action_{\idxiter}} \action_{\idxiter}^\top\shat{\kernmat} \action_j
	\end{align*}
	Now if \(j = \idxiter\), clearly \(\action_{\idxiter+1}^\top \shat{\kernmat} \action_{j}
	=
	\action_{\idxiter+1}^\top \shat{\kernmat} \action_{\idxiter} = 0\). If \(j < \idxiter\), we have
	using \eqref{eqn:proof-cg-ih-span-equiv}, that
	\begin{equation}
		\label{eqn:proof-cg-precond-observ-span-inclusion}
		\shat{\precond}^{-1}\shat{\kernmat} \action_j \in \langle \shat{\precond}^{-1}\shat{\kernmat} \residual_0,
		\dots, (\shat{\precond}^{-1}\shat{\kernmat})^{j}\residual_0 \rangle
		\subset \langle \shat{\precond}^{-1}\residual_0, \dots,
		\shat{\precond}^{-1}\residual_{j}\rangle.
	\end{equation}
	Therefore we obtain for \(j < \idxiter\), that
	\begin{equation}
		\label{eqn:proof-cg-istep-residual-action-conjugacy}
		\action_{\idxiter+1}^\top \shat{\kernmat} \action_{j} \overset{\eqref{eqn:proof-cg-ih-conjugacy} }{=}
		\residual_\idxiter^\top \shat{\precond}^{-1} \shat{\kernmat} \action_j \overset{\eqref{eqn:proof-cg-precond-observ-span-inclusion}}{=}
		\residual_\idxiter^\top \bigg(\sum_{\ell=1}^{j} \gamma_\ell \shat{\precond}^{-1} \residual_\ell\bigg) \overset{\eqref{eqn:proof-cg-ih-resid-ortho}}{=} 0.
	\end{equation}
	Thus in combination we have
	\begin{equation}
		\label{eqn:proof-cg-istep-conjugacy}
		\forall j \in \{1, \dots, \idxiter\}: \quad \action_{\idxiter+1}^\top \shat{\kernmat}
		\action_{j} = 0.
	\end{equation}
	Now for the search direction we have
	\begin{equation}
		\label{eqn:proof-cg-istep-searchdir}
		\begin{aligned}
			\searchdir_{\idxiter + 1} & = \qoicov_\idxiter \shat{\kernmat} \action_{\idxiter + 1} = \bigg(\qoicov_0 - \sum_{j=1}^{\idxiter} \frac{\searchdir_j \searchdir_j^\top}{\searchdirsqnorm_j} \bigg)\shat{\kernmat} \action_{\idxiter+1}                                                                                                                             \\
			                          & =\action_{\idxiter +1} -  \sum_{j=1}^{\idxiter} \frac{ \searchdir_j^\top\shat{\kernmat}\action_{\idxiter+1}}{\searchdirsqnorm_j}\searchdir_j \overset{\eqref{eqn:action-equivalence-cg}}{=} \action_{\idxiter +1} -  \sum_{j=1}^{\idxiter} \frac{ \action_j^\top\shat{\kernmat}\action_{\idxiter+1}}{\searchdirsqnorm_j}\searchdir_j \\
			                          & \overset{\eqref{eqn:proof-cg-istep-conjugacy}}{=}\action_{\idxiter +1}.
		\end{aligned}
	\end{equation}
	Further, we have for the solution estimate, that \(\qoimean_{\idxiter+1} = \qoimean_\idxiter +
	\searchdir_{\idxiter+1}\frac{\observ_{\idxiter+1}}{\searchdirsqnorm_{\idxiter+1}}\). It holds that
	\begin{align*}
		\observ_{\idxiter+1}          & = \action_{\idxiter+1}^\top \residual_{\idxiter} = \big( \shat{\precond}^{-1} \residual_{\idxiter} -
		\frac{(\shat{\precond}^{-1}\residual_{\idxiter})^\top \shat{\kernmat} \action_{\idxiter}}{\action_{\idxiter}^\top\shat{\kernmat} \action_{\idxiter}} \action_{\idxiter}\big)^\top \residual_\idxiter                                                                                                                                                                                                                      \\
		                              & = (\shat{\precond}^{-1}\residual_{\idxiter})^\top \residual_{\idxiter} - \sum_{j=}^{\idxiter}c_j(\shat{\precond}^{-1}\residual_{j-1})^\top \residual_\idxiter \overset{\eqref{eqn:proof-cg-ih-resid-ortho}}{=} (\shat{\precond}^{-1}\residual_{\idxiter})^\top \residual_\idxiter
		\intertext{as well as}
		\searchdirsqnorm_{\idxiter+1} & = \action_{\idxiter+1}^\top \shat{\kernmat} \qoicov_{\idxiter} \shat{\kernmat} \action_{\idxiter+1} = \searchdir_{\idxiter+1}^\top \shat{\kernmat} \action_{\idxiter+1} \overset{\eqref{eqn:proof-cg-istep-searchdir}}{=} \action_{\idxiter+1}^\top\shat{\kernmat} \action_{\idxiter+1}
	\end{align*}
	Combining the above and recalling \Cref{alg:cg}, we obtain
	\begin{equation*}
		\qoimean_{\idxiter+1} = \qoimean_{\idxiter} + \searchdir_{\idxiter+1}\frac{\observ_{\idxiter+1}}{\searchdirsqnorm_{\idxiter+1}}
		= \qoimean_{\idxiter} + \searchdir_{\idxiter+1}\frac{(\shat{\precond}^{-1}\residual_\idxiter)^\top \residual_\idxiter}{\action_{\idxiter+1}^\top\shat{\kernmat} \action_{\idxiter+1}} =
		\qoimean_{\idxiter+1}^\txtCG.
	\end{equation*}
	Finally, the residual is computed identically in \Cref{alg:itergp} as in
	\Cref{alg:cg}, giving
	\begin{equation*}
		\residual_\idxiter = (\labels - \vmu) - \shat{\kernmat} \qoimean_\idxiter = (\labels - \vmu) - \shat{\kernmat}
		\qoimean_{\idxiter}^\txtCG = \residual_\idxiter^\txtCG.
	\end{equation*}
	This proves the claims.
\end{proof}

\begin{corollary}[Preconditioned Gradient Actions as CG Actions]
	\label{cor:itergp-cg-gradient-actions}
	Choosing actions
	\begin{equation}
		\action_{\idxiter} = \shat{\precond}^{-1} \residual_{\idxiter-1}
	\end{equation}
	in \Cref{thm:itergp-cg} instead also reproduces the preconditioned conjugate gradient method,
	i.e. it holds for      \(\idxiter \in \{1,   \dots,  n\}\) that
	\begin{align}
		\searchdir_\idxiter    & = \action_\idxiter^\txtCG        \\
		\qoimean_\idxiter      & = \qoimean_\idxiter^{\txtCG}     \\
		\residual_{\idxiter-1} & = \residual_{\idxiter-1}^\txtCG.
	\end{align}
\end{corollary}

\begin{proof}
	It suffices to show that \(\searchdir_\idxiter = \action_\idxiter^\txtCG\). The rest of the
	argument is then identical to the proof of \Cref{thm:itergp-cg}. We prove the claim by
	induction. For the base case by assumption \(\action_1 =
	\shat{\precond}^{-1}\residual_{0} = \action_1^{\txtCG}\). Now for the induction
	step \(\idxiter \to \idxiter +1\), assume that \(\searchdir_j = \action_j\) for all \(j \leq i\),
	then
	\begin{align*}
		\searchdir_{\idxiter+1} & = \qoicov_{\idxiter} \shat{\kernmat} \shat{\precond}^{-1}\residual_{\idxiter}                                                                                                                                                                                                          \\
		                        & = (\mI - \invapprox_\idxiter \shat{\kernmat})\shat{\precond}^{-1}\residual_\idxiter                                                                                                                                                                                                    \\
		                        & = \shat{\precond}^{-1}\residual_\idxiter - \mSearchdir_\idxiter (\mSearchdir_\idxiter^\top \shat{\kernmat} \mSearchdir_\idxiter)^{-1} \mSearchdir_{\idxiter}^\top \shat{\kernmat} \shat{\precond}^{-1}\residual_\idxiter                         & \text{By \cref{eqn:invapprox-batch-form}.} \\
		                        &\overset{\mathrm{IH}}{=} \shat{\precond}^{-1}\residual_\idxiter - \mActions^\txtCG_\idxiter ((\mActions^\txtCG_\idxiter)^\top \shat{\kernmat} \mActions^\txtCG_\idxiter)^{-1} (\mActions^\txtCG_{\idxiter})^\top \shat{\kernmat} \shat{\precond}^{-1}\residual_\idxiter               \\
		\intertext{Now by the same argument as in \cref{eqn:proof-cg-istep-residual-action-conjugacy} in the proof of \Cref{thm:itergp-cg} we
			have for all \(j < \idxiter\) that \(\residual_\idxiter^\top \shat{\precond}^{-1}
			\shat{\kernmat} \action_j^\txtCG = 0\). Therefore}
		                        & = \shat{\precond}^{-1}\residual_\idxiter - \action^\txtCG_\idxiter ((\action^\txtCG_\idxiter)^\top \shat{\kernmat} \action^\txtCG_\idxiter)^{-1} (\action^\txtCG_{\idxiter})^\top \shat{\kernmat} \shat{\precond}^{-1}\residual_\idxiter                                                      \\
		                        & = \action_{\idxiter+1}^\txtCG                                                                                                                                                                                                      & \text{By \cref{eqn:cg-actions}.}
	\end{align*}
	This proves the claim.
\end{proof}

\begin{corollary}[Deflated Conjugate Gradient Method]
	\label{cor:itergp-deflated-cg}
	Let the first \(0 < \ell < n\) actions \((\action_\idxiter)_{\idxiter=1}^\ell\) of \Cref{alg:itergp} be linearly independent and the
	remaining ones be given by \(\action_\idxiter = \shat\precond^{-1}\residual_\idxiter\), where \(\shat\precond \approx \shat\kernmat\) is a preconditioner. Then \Cref{alg:itergp}
	is equivalent to the preconditioned deflated CG algorithm \cite[Alg.~3.6]{Saad2000} with deflation subspace
	\(\linspan{\mActions_\ell}\).
\end{corollary}
\begin{proof}
	By the form of preconditioned deflated CG given in Algorithm 3.6 of \citet{Saad2000} and \Cref{cor:itergp-cg-gradient-actions}, it suffices to show that the residual \(\residual_{\ell}\) satisfies \(\mS_\ell^\top \residual_\ell = \vzero\) and that for all \(\idxiter > \ell\), it holds that
	\begin{equation*}
		\action_{\idxiter}^{\mathrm{defCG}} = \searchdir_\idxiter = (\mI - \invapprox_{\idxiter-1} \shat{\kernmat}) \action_\idxiter.
	\end{equation*}
	Now it holds by \Cref{lem:relation-precision-matrix-qoimean} and \cref{eqn:invapprox-batch-form}, that
	\begin{equation*}
		\mS_\ell^\top \residual_\ell = \mS_\ell^\top(\mI - \shat{\kernmat}\invapprox_{\ell})(\labels - \vmu)= \underbracket[0.14ex]{\mS_\ell^\top(\mI - \shat{\kernmat} \mActions_\ell(\mActions_\ell^\top \shat{\kernmat} \mActions_\ell)^{-1} \mActions_\ell^\top)}_{=\mZero}(\labels - \vmu) = \vzero.
	\end{equation*}
	This proves the first claim. Now, by \citet[Alg.~3.6]{Saad2000}, the search directions \((\action_\idxiter^{\mathrm{defCG}})_{\idxiter=\ell + 1}^n\) of preconditioned deflated CG are given by
	\begin{align*}
		\action_{\idxiter}^{\mathrm{defCG}} &= \action_{\idxiter}^{\txtCG} - \mActions_\ell(\mActions_\ell^\top \shat{\kernmat} \mActions_\ell)^{-1}\mActions_\ell^\top \shat{\kernmat} \shat\precond^{-1}\residual_\idxiter\\
		&= (\mI - \invapprox_{\ell+1:(\idxiter-1)}\shat{\kernmat}) \action_\idxiter - \mActions_\ell(\mActions_\ell^\top \shat{\kernmat} \mActions_\ell)^{-1}\mActions_\ell^\top \shat{\kernmat} \shat\precond^{-1}\residual_\idxiter & \text{\Cref{cor:itergp-cg-gradient-actions}}\\
		&= (\mI - \invapprox_{\ell+1:(\idxiter-1)}\shat{\kernmat}) \action_\idxiter - \invapprox_{\ell} \shat{\kernmat} \action_\idxiter\\
		&= (\mI - (\invapprox_{\ell+1:(\idxiter-1)} - \invapprox_{\ell})\shat{\kernmat})\action_\idxiter\\
		&= (\mI - \invapprox_{\idxiter-1}\shat\kernmat)\action_\idxiter\\
		&= \searchdir_\idxiter
	\end{align*}
	This proves the claim.
\end{proof}

\begin{remark}[Preconditioning and \Cref{alg:itergp}]

Iterative methods typically have convergence rates depending on the condition number of the system
matrix. One successful strategy in practice to accelerate convergence is to use a preconditioner
\(\shat{\precond} \approx \shat{\kernmat}\) \cite{Nocedal2006}. A
preconditioner needs to be cheap to compute and
allow efficient matrix-vector multiplication \(\vv \mapsto \shat\precond^{-1}\vv\). Now,
\Cref{alg:itergp} implicitly constructs and applies a \emph{deflation-based preconditioner}, which are
defined via
a deflation subspace to be projected out \cite{Frank2001}. In \Cref{alg:itergp} this
is precisely the
already explored space \(\linspan{\mActions_\idxiter} = \linspan{\mSearchdir_\idxiter}\) spanned by the
actions. Therefore, if we run a mixed strategy, meaning first choosing actions that define a
certain subspace and then choose residual actions, we recover the \emph{deflated conjugate gradient method}
\cite{Saad2000} (see \Cref{cor:itergp-deflated-cg} for a proof).
Alternatively, one can also use byproducts of the iteration of \Cref{alg:itergp} to
construct
a diagonal-plus-low-rank preconditioner of the form \(\shat{\precond} = \sigma^2 \mI + \mU
\mU^\top
\approx \shat{\kernmat}\) where \(\mU=\kernmat \mSearchdir_\idxiter
\operatorname{diag}(\eta_1, \dots, \eta_\idxiter) \in \R^{n \times \idxiter}\). Therefore, again if running a mixed
strategy, one can first construct
a preconditioner and then accelerate the convergence of subsequent CG iterations. In this sense one
can double-dip in terms of preconditioning (conjugate) gradient iterations by
combining these two techniques \emph{at essentially no overhead}.
\end{remark}

\subsection{Inducing Point Methods}

\begin{theorem}[Approximate Posterior Mean of Nyström, SoR, DTC and SVGP]
	\label{thm:itergp-nystroem}
	Let \(\inducingpoints \in \R^{n \times m}\) be a set of distinct inducing inputs such that
	\(\rank(\kernmat_{\traindata \inducingpoints}) = m \leq n\). Then the posterior mean of the Nyström variants subset of regressors (SoR) and deterministic training conditional (DTC) is identical to the one of SVGP and given by
	\begin{equation}
		\begin{aligned}
			\mu(\cdot) & = k(\cdot, \inducingpoints) (\kernmat_{\inducingpoints \traindata} \kernmat_{\traindata \inducingpoints} + \sigma^2 \kernmat_{\inducingpoints \inducingpoints})^{-1} \kernmat_{\inducingpoints \traindata} (\labels - \vmu)                    \\
			           & = q(\cdot, \traindata) \kernmat_{\traindata \inducingpoints}(\kernmat_{\inducingpoints \traindata} (q(\traindata, \traindata) + \sigma^2 \mI)\kernmat_{\traindata\inducingpoints })^{-1} \kernmat_{\inducingpoints \traindata}(\labels - \vmu)
		\end{aligned}
	\end{equation}
\end{theorem}

\begin{proof}
	First, note that by eqns. (16b) and (20b) of \citet{QuinoneroCandela2005} the posterior mean of SoR and
	DTC is identical and given by
	\begin{equation*}
		\mu(\cdot) = k(\cdot, \inducingpoints) (\kernmat_{\inducingpoints \traindata}
		\kernmat_{\traindata \inducingpoints} + \sigma^2 \kernmat_{\inducingpoints \inducingpoints})^{-1}
		\kernmat_{\inducingpoints \traindata} (\labels - \vmu)
	\end{equation*}
	Now, by Theorem 5 of \citet{Wild2021} the posterior mean of SVGP for a fixed set of inducing points is
	equivalent to the Nyström approximation, which takes the form above. Recognizing that
	\(\kernmat_{\inducingpoints \traindata} \kernmat_{\traindata \inducingpoints} \in \R^{m \times m}\) is invertible, it
	holds that
	\begin{align*}
		\mu(\cdot) & = k(\cdot, \inducingpoints) (\kernmat_{\inducingpoints \traindata} \kernmat_{\traindata \inducingpoints} + \sigma^2 \kernmat_{\inducingpoints \inducingpoints})^{-1} \kernmat_{\inducingpoints \traindata} (\labels - \vmu)                                                                                                                                                                                                                                                                                                   \\
		           & = k(\cdot, \inducingpoints) \kernmat_{\inducingpoints \inducingpoints}^{-1} (\kernmat_{\inducingpoints \traindata} \kernmat_{\traindata \inducingpoints}\kernmat_{\inducingpoints \inducingpoints}^{-1} + \sigma^2 \mI )^{-1} \kernmat_{\inducingpoints \traindata} (\labels - \vmu)                                                                                                                                                                                                                                            \\
		           & = k(\cdot, \inducingpoints) \kernmat_{\inducingpoints \inducingpoints}^{-1} \kernmat_{\inducingpoints \traindata} \kernmat_{\traindata\inducingpoints } (\kernmat_{\inducingpoints \traindata} \kernmat_{\traindata\inducingpoints }\kernmat_{\inducingpoints \inducingpoints}^{-1}\kernmat_{\inducingpoints \traindata} \kernmat_{\traindata \inducingpoints} + \sigma^2 \kernmat_{\inducingpoints \traindata} \kernmat_{\traindata\inducingpoints } )^{-1} \kernmat_{\inducingpoints \traindata} (\labels - \vmu) \\
		           & = k(\cdot, \inducingpoints) \kernmat_{\inducingpoints \inducingpoints}^{-1}\kernmat_{\inducingpoints \traindata} \kernmat_{\traindata\inducingpoints } (\kernmat_{\inducingpoints \traindata} (\kernmat_{\traindata\inducingpoints }\kernmat_{\inducingpoints \inducingpoints}^{-1}\kernmat_{\inducingpoints \traindata} + \sigma^2 \mI)\kernmat_{\traindata\inducingpoints })^{-1} \kernmat_{\inducingpoints \traindata}(\labels - \vmu)                                                                                     \\
		           & = q(\cdot, \traindata) \kernmat_{\traindata \inducingpoints}(\kernmat_{\inducingpoints \traindata} (q(\traindata, \traindata) + \sigma^2 \mI)\kernmat_{\traindata\inducingpoints })^{-1} \kernmat_{\inducingpoints \traindata}(\labels - \vmu)
	\end{align*}
	This proves the claim.
\end{proof}

\section{Theoretical Results and Proofs}

\subsection{Properties of Algorithm 1}

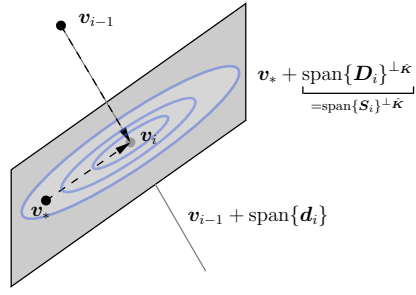
\begin{figure}
	\centering
	\scalebox{0.65}{
		\tikzset{every picture/.style={line width=0.75pt}} 

\begin{tikzpicture}[
        x=0.75pt,
        y=0.75pt,
        yscale=-1,
        xscale=1,
    ]

    \draw  [fill={rgb, 255:red, 0; green, 0; blue, 0 }  ,fill opacity=0.2 ] (36.37,218.7) -- (36.37,128) -- (218.37,-1.09) -- (218.37,89.61) -- cycle ;
    \draw [color={rgb, 255:red, 128; green, 128; blue, 128 }  ,draw opacity=1 ][fill={rgb, 255:red, 0; green, 0; blue, 0 }  ,fill opacity=0.2 ]   (75,16) -- (128.92,106.76) ;
    \draw [shift={(128.92,106.76)}, rotate = 59.29] [color={rgb, 255:red, 128; green, 128; blue, 128 }  ,draw opacity=1 ][fill={rgb, 255:red, 128; green, 128; blue, 128 }  ,fill opacity=1 ][line width=0.75]      (0, 0) circle [x radius= 3.35, y radius= 3.35]	;
    \draw [color={rgb, 255:red, 128; green, 128; blue, 128 }  ,draw opacity=1 ][fill={rgb, 255:red, 0; green, 0; blue, 0 }  ,fill opacity=0.2 ]   (148,139) -- (187.37,206.85) ;

    \draw  [color={rgb, 255:red, 74; green, 108; blue, 226 }  ,draw opacity=0.5 ][fill={rgb, 255:red, 255; green, 255; blue, 255 }  ,fill opacity=0.2 ][line width=1.5]  (45.3,163.6) .. controls (39.06,154.41) and (71.43,121.51) .. (117.6,90.12) .. controls
    (163.78,58.73) and (206.28,40.73) .. (212.53,49.92) .. controls (218.78,59.11) and (186.41,92.01)
    .. (140.23,123.4) .. controls (94.05,154.8) and (51.55,172.8) .. (45.3,163.6) -- cycle ;
    \draw  [dash pattern={on 4.5pt off 4.5pt}]  (75,16) -- (127.9,105.04) ;
    \draw [shift={(128.92,106.76)}, rotate = 239.29] [fill={rgb, 255:red, 0; green, 0; blue, 0 }  ][line width=0.08]  [draw opacity=0] (12,-3) -- (0,0) -- (12,3) -- cycle    ;
    \draw [shift={(75,16)}, rotate = 59.29] [color={rgb, 255:red, 0; green, 0; blue, 0 }  ][fill={rgb, 255:red, 0; green, 0; blue, 0 }  ][line width=0.75]      (0, 0) circle [x radius= 3.35, y radius= 3.35]	;
    \draw  [dash pattern={on 4.5pt off 4.5pt}]  (63.37,151.85) -- (127.27,107.89) ;
    \draw [shift={(128.92,106.76)}, rotate = 505.48] [fill={rgb, 255:red, 0; green, 0; blue, 0 }  ][line width=0.08]  [draw opacity=0] (12,-3) -- (0,0) -- (12,3) -- cycle    ;
    \draw [shift={(63.37,151.85)}, rotate = 325.48] [color={rgb, 255:red, 0; green, 0; blue, 0 }  ][fill={rgb, 255:red, 0; green, 0; blue, 0 }  ][line width=0.75]      (0, 0) circle [x radius= 3.35, y radius= 3.35]	 ;
    \draw  [color={rgb, 255:red, 74; green, 108; blue, 226 }  ,draw opacity=0.5 ][line width=1.5]  (99.97,126.44) .. controls (97.81,123.26) and (109.01,111.87) .. (125,101) .. controls
    (140.99,90.13) and (155.7,83.9) .. (157.86,87.08) .. controls (160.02,90.26) and (148.82,101.65) ..
    (132.83,112.52) .. controls (116.85,123.39) and (102.14,129.62) .. (99.97,126.44) -- cycle ;
    \draw  [color={rgb, 255:red, 74; green, 108; blue, 226 }  ,draw opacity=0.5 ][line width=1.5]  (75.88,143.81) .. controls (72.03,138.15) and (91.96,117.89) .. (120.4,98.56) .. controls
    (148.84,79.23) and (175.01,68.14) .. (178.85,73.8) .. controls (182.7,79.46) and (162.77,99.72) ..
    (134.33,119.05) .. controls (105.9,138.38) and (79.73,149.47) .. (75.88,143.81) -- cycle ;

    \draw (226,42.4) node [anchor=north west][inner sep=0.75pt]  [font=\large]  {$\qoi + \underbracket[0.14ex]{\linspan{\mSearchdir_\idxiter}^{\perp_{\shat{\kernmat}}}}_{=\linspan{\mActions_\idxiter}^{\perp_{\shat{\kernmat}}}}$};
    \draw (172,154.4) node [anchor=north west][inner sep=0.75pt]  [font=\large]  {$\qoimean_{\idxiter -1} + \linspan{\searchdir_\idxiter}$};
    \draw (135,97) node [anchor=north west][inner sep=0.75pt]  [font=\large]  {$\qoimean_{\idxiter}$};
    \draw (50,155.4) node [anchor=north west][inner sep=0.75pt]  [font=\large]  {$\qoi$};
    \draw (87,5.4) node [anchor=north west][inner sep=0.75pt]  [font=\large]  {$\qoimean_{\idxiter-1}$};

\end{tikzpicture}
	}
	\caption{\emph{Geometric perspective on the probabilistic linear solver learning representer weights \(\qoi\).}}
\end{figure}

\begin{lemma}[Geometric Properties of \Cref{alg:itergp}]
	\label{lem:geometric-properties}
	Let \(i \in \{1, \dots, n\}\), and assume \(\qoicov_0\) is chosen such that \(\qoicov_0
	\shat{\kernmat} \action_j = \action_j\) for all \(j\leq \idxiter\) (e.g. \(\qoicov_0 =
	\shat{\kernmat}^{-1}\)). Then it holds for the quantities computed by \Cref{alg:itergp}
	that
	\begin{align}
		 & \linspan{\mActions_\idxiter} = \linspan{\mSearchdir_\idxiter} \label{eqn:action-searchdir-span}                                                                                                                                                                                          \\
		 & \invapprox_\idxiter = \mSearchdir_\idxiter(\mSearchdir_\idxiter^\top \shat{\kernmat} \mSearchdir_\idxiter)^{-1} \mSearchdir_\idxiter^\top = \mActions_\idxiter(\mActions_\idxiter^\top \shat{\kernmat} \mActions_\idxiter)^{-1} \mActions_\idxiter^\top \label{eqn:invapprox-batch-form} \\
		 & \text{\(\invapprox_\idxiter \shat{\kernmat}\) is the \(\shat{\kernmat}\)-orthogonal
			projection onto \(\linspan{\mSearchdir_\idxiter}\)} \label{eqn:ci-projection}                                                                                                                                                                                                               \\
		 & \text{\(\qoicov_\idxiter \shat{\kernmat}\) is the \(\shat{\kernmat}\)-orthogonal projection onto \(\linspan{\mSearchdir_\idxiter}^{\perp_{\shat{\kernmat}}}\)} \label{eqn:sigmai-projection}                                                                                             \\
		 & \searchdir_i^\top \shat{\kernmat} \searchdir_j = 0 \qquad \text{for all \(j < i\)} \label{eqn:conjugate-directions}
	\end{align}
	where \(\mActions_\idxiter = \begin{pmatrix} \action_1 \cdots \action_\idxiter \end{pmatrix} \in \R^{n \times \idxiter}\) and
	\(\mSearchdir_\idxiter = \begin{pmatrix} \searchdir_1 \cdots \searchdir_\idxiter \end{pmatrix} \in \R^{n \times \idxiter}\).
\end{lemma}

\begin{proof}
	We prove the claims by induction. We begin with the base case \(\idxiter=1\).

	By assumption it holds that \(\mActions_1 = \action_1 = \qoicov_0 \shat{\kernmat} \action_1
	= \searchdir_1 = \mSearchdir_1\). Now by \Cref{alg:itergp}, we have \(\invapprox_1 =
	\frac{1}{\eta_1}\searchdir_1 \searchdir_1^\top\), which with the above proves
	\eqref{eqn:invapprox-batch-form}. By the batched form \eqref{eqn:invapprox-batch-form} of \(\invapprox_\idxiter\),
	the statements \eqref{eqn:ci-projection} and \eqref{eqn:sigmai-projection} follow immediately.
	Finally, \(\shat{\kernmat}\)-orthogonality for a single search direction holds trivially.

	Now for the induction step \(\idxiter \to \idxiter + 1\). Assume that \cref{eqn:action-searchdir-span,eqn:invapprox-batch-form,eqn:ci-projection,eqn:sigmai-projection,eqn:conjugate-directions} hold
	for iteration \(\idxiter\). Then we have that
	\begin{equation*}
		\searchdir_{\idxiter + 1} = \qoicov_{\idxiter} \shat{\kernmat}
		\action_{\idxiter + 1}=\action_{\idxiter + 1} - \invapprox_{\idxiter}\shat{\kernmat}
		\action_{\idxiter + 1} \overset{\eqref{eqn:invapprox-batch-form}}{=}\action_{\idxiter + 1} -
		\mActions_\idxiter(\mActions_\idxiter^\top \shat{\kernmat} \mActions_\idxiter)^{-1} \mActions_\idxiter^\top \shat{\kernmat}
		\action_{\idxiter + 1} \in \linspan{\mActions_{\idxiter+1}}
	\end{equation*}
	By the induction hypothesis the above also implies \(\linspan{\mActions_{\idxiter+1}} =
	\linspan{\mSearchdir_{\idxiter+1}}\). This proves \cref{eqn:action-searchdir-span}. Next, we have by the induction
	hypotheses \eqref{eqn:invapprox-batch-form} and \eqref{eqn:conjugate-directions} that
	\begin{align*}
		\invapprox_{\idxiter + 1} & = \invapprox_{\idxiter} + \frac{1}{\searchdirsqnorm}\searchdir_{\idxiter+1} \searchdir_{\idxiter+1}^\top                                                                                                            \\
		                          & = \mSearchdir_\idxiter(\mSearchdir_\idxiter^\top \shat{\kernmat} \mSearchdir_\idxiter)^{-1} \mSearchdir_\idxiter^\top + \frac{1}{\searchdirsqnorm_{\idxiter+1}}\searchdir_{\idxiter+1} \searchdir_{\idxiter+1}^\top \\
		                          & = \sum_{k=1}^{\idxiter+1} \frac{1}{\searchdirsqnorm_{k}}\searchdir_{k} \searchdir_{k}^\top                                                                                                                          \\
		                          & = \mSearchdir_{\idxiter+1}(\mSearchdir_{\idxiter+1}^\top \shat{\kernmat} \mSearchdir_{\idxiter+1})^{-1} \mSearchdir_{\idxiter+1}^\top
	\end{align*}
	This proves the first equality of \cref{eqn:invapprox-batch-form}. For the second, first recognize that an
	orthogonal projection onto a linear subspace \(\linspan{\mA}\) with respect to the
	\(\mB\)-inner product is given by \(\mP_{\mA} = \mA(\mA^\top \mB \mA)^{-1}\mA^\top
	\mB\). The projection onto its \(\mB\)-orthogonal subspace is given by \(\mP_{\mA^{\perp_{\mB}}} =
	\mI - \mP_\mA\). Therefore \cref{eqn:ci-projection,eqn:sigmai-projection} follow directly from the above argument. Now
	since projection onto a subspace is unique and independent of the choice of basis,
	we have by \(\linspan{\mSearchdir_{\idxiter+1}} = \linspan{\mActions_{\idxiter+1}}\) that
	\begin{align*}
		\invapprox_\idxiter \shat{\kernmat} & = \mP_{\mSearchdir_{\idxiter+1}} =\mP_{\mActions_{\idxiter+1}} =\mActions_\idxiter(\mActions_\idxiter^\top \shat{\kernmat} \mActions_\idxiter)^{-1} \mActions_\idxiter^\top \shat{\kernmat}
	\end{align*}
	Now since \(\shat{\kernmat}\) is non-singular, the second equality of
	\cref{eqn:invapprox-batch-form} follows. Finally, we will prove \(\shat{\kernmat}\)-orthogonality
	of the search directions. Let \(j<\idxiter+1\), then it holds that
	\begin{align*}
		\searchdir_{\idxiter+1}^\top \shat{\kernmat} \searchdir_j = (\underbracket[0.14ex]{\qoicov_{\idxiter}\shat{\kernmat}\action_{\idxiter+1}}_{\in \linspan{\mActions_\idxiter}^{\perp_{\shat{\kernmat}}}})^\top  \shat{\kernmat} \underbracket[0.14ex]{\searchdir_j}_{\in \linspan{\mActions_\idxiter}} = 0
	\end{align*}
	by \cref{eqn:action-searchdir-span,eqn:sigmai-projection}. This completes the proof.
\end{proof}

\begin{corollary}
	\label{cor:k-ortho-proj-properties}
	Let \(\idxiter \in \{1, \dots, n\}\). It holds for \(\invapprox_\idxiter \shat{\kernmat}\),
	the \(\shat{\kernmat}\)-orthogonal projection onto \(\mActions_\idxiter\), that
	\begin{align}
		(\invapprox_{\idxiter} \shat{\kernmat})^2                   & = \invapprox_{\idxiter} \shat{\kernmat} \label{eqn:ci-projection-property} \\
		\invapprox_{\idxiter} \shat{\kernmat} \invapprox_{\idxiter} & = \invapprox_{\idxiter} \label{eqn:ci-k-ci}
	\end{align}
	Further for \(\mH_\idxiter=\qoicov_\idxiter \shat{\kernmat} = \mI - \invapprox_\idxiter
	\shat{\kernmat}\) the \(\shat{\kernmat}\)-orthogonal projection onto
	\(\mActions_\idxiter^{\perp_{\shat{\kernmat}}}\), we have
	\begin{align}
		\mH_\idxiter^2                                 & = \mH_\idxiter                                                                        \\
		\mH_\idxiter^\top \shat{\kernmat} \mH_\idxiter & = \mH_\idxiter^\top \shat{\kernmat} = \shat{\kernmat} \mH_\idxiter\label{eqn:hi-k-hi}
	\end{align}
\end{corollary}

\begin{proof}
	By \Cref{lem:geometric-properties}, it holds that \(\invapprox_\idxiter = \mActions_\idxiter (\mActions_\idxiter^\top \shat{\kernmat} \mActions_\idxiter)^{-1}
	\mActions_\idxiter^\top\). Therefore
	\begin{align*}
		\invapprox_{\idxiter} \shat{\kernmat} \invapprox_{\idxiter} & = \mActions_\idxiter (\mActions_\idxiter^\top \shat{\kernmat} \mActions_\idxiter)^{-1} \mActions_\idxiter^\top \shat{\kernmat} \mActions_\idxiter (\mActions_\idxiter^\top \shat{\kernmat} \mActions_\idxiter)^{-1} \mActions_\idxiter^\top = \invapprox_\idxiter.
	\end{align*}
	This proves \eqref{eqn:ci-k-ci} and \eqref{eqn:ci-projection-property}.
	Define \(\mH_\idxiter = \mI - \invapprox_\idxiter \shat{\kernmat}\), then
	\begin{align*}
		\mH_\idxiter \mH_\idxiter                      & = (\mI - \invapprox_\idxiter \shat{\kernmat})(\mI - \invapprox_\idxiter \shat{\kernmat})
		= \mI - 2 \invapprox_\idxiter \shat{\kernmat} + (\invapprox_\idxiter \shat{\kernmat})^2
		= \mI - \invapprox_\idxiter \shat{\kernmat}= \mH_\idxiter
		\intertext{as well as}
		\mH_\idxiter^\top \shat{\kernmat} \mH_\idxiter & = (\mI - \invapprox_\idxiter \shat{\kernmat})^\top \shat{\kernmat} (\mI - \invapprox_\idxiter \shat{\kernmat}) = (\shat{\kernmat} - \shat{\kernmat} \invapprox_\idxiter \shat{\kernmat})(\mI - \invapprox_\idxiter \shat{\kernmat}) \\
		                                               & =\shat{\kernmat} - 2\shat{\kernmat} \invapprox_\idxiter \shat{\kernmat} + \shat{\kernmat} (\invapprox_\idxiter \shat{\kernmat})^2                                                                                                   \\
		                                               & = \shat{\kernmat} - \shat{\kernmat} \invapprox_\idxiter \shat{\kernmat} = \mH_\idxiter^\top \shat{\kernmat} = \shat{\kernmat} \mH_\idxiter.
	\end{align*}

\end{proof}

\begin{lemma}
	\label{lem:relation-precision-matrix-qoimean}
	Let \(\qoicov_0 = \shat{\kernmat}^{-1}\), then it holds
	that
	\begin{align}
		\invapprox_{\idxiter} (\labels - \vmu) & = \qoimean_{\idxiter},       \label{eqn:relation-precision-matrix-qoimean} \\
		\qoicov_{\idxiter}(\labels - \vmu)     & = \qoi - \qoimean_{\idxiter}.\label{eqn:relation-cov-matrix-qoimean}
	\end{align}
\end{lemma}
\begin{proof}
	We prove the statement by induction. By assumption \(\invapprox_{0} (\labels - \vmu) =
	\qoimean_{0}\). Now assume \eqref{eqn:relation-precision-matrix-qoimean} holds. Then for \(\idxiter \to
	\idxiter + 1\), we have
	\begin{align*}
		\invapprox_{\idxiter + 1} (\labels - \vmu)     & = (\invapprox_{\idxiter} + \frac{1}{\searchdirsqnorm_{\idxiter +1}} \searchdir_{\idxiter +1} \searchdir_{\idxiter +1}^\top) (\labels - \vmu) \overset{\textrm{IH}}{=} \qoimean_{\idxiter} + \frac{\searchdir_{\idxiter +1}^\top (\labels - \vmu)}{\searchdirsqnorm_{\idxiter +1}} \searchdir_{\idxiter +1} \\
		\intertext{Now by the update to the representer weights in \Cref{alg:itergp} it suffices to show that
			\(\observ_{\idxiter+1} = \searchdir_{\idxiter +1}^\top (\labels - \vmu)\). We have}
		\searchdir_{\idxiter +1}^\top (\labels - \vmu) & = (\qoicov_{\idxiter}\shat{\kernmat} \action_{\idxiter+1})^\top (\labels - \vmu) = \action_{\idxiter+1}^\top \shat{\kernmat} \qoicov_{\idxiter} (\labels - \vmu)                                                                                                                                           \\
		                                               & = \action_{\idxiter+1}^\top \shat{\kernmat} (\shat{\kernmat}^{-1} - \invapprox_{\idxiter})(\labels - \vmu) \overset{\textrm{IH}}{=} \action_{\idxiter+1}^\top((\labels - \vmu) - \shat{\kernmat} \qoimean_{\idxiter}) = \action_{\idxiter+1}^\top \residual_{\idxiter} = \observ_\idxiter.
	\end{align*}
\end{proof}

\begin{lemma}
	Let \(\qoicov_0 = \shat{\kernmat}^{-1}\), \(\invapprox_0 = \mZero\) and consequently
	\(\qoimean_0 = \vzero\), then it holds for the residual at iteration \(\idxiter \in \{1, \dots, n\}\) that
	\begin{align}
		\residual_{\idxiter-1} & = \shat{\kernmat}(\qoi - \qoimean_{\idxiter-1})                                                                                                              \\
		                                                                                   & = \shat{\kernmat}\qoicov_{\idxiter-1} \shat{\kernmat} \qoi\\
																						   &= (\shat{\kernmat} -
		\kernmatapprox_{\idxiter-1})\qoi.
	\end{align}
\end{lemma}

\begin{proof}
	It holds by definition, that
	\begin{align*}
		\residual_{\idxiter-1} & = (\labels - \vmu) - \shat{\kernmat}\qoimean_{\idxiter-1} = \shat{\kernmat} \qoi - \shat{\kernmat} \qoimean_{\idxiter -1} = \shat\kernmat(\qoi - \qoimean_{\idxiter-1}).
		\intertext{Further we have by \cref{eqn:relation-cov-matrix-qoimean}, that}
		&= \shat\kernmat \qoicov_{\idxiter-1}(\labels - \vmu) = \shat{\kernmat}\qoicov_{\idxiter-1} \shat{\kernmat} \qoi,
		\intertext{and finally, by the definition of the kernel matrix approximation in \Cref{alg:itergp}, we obtain}
		&= \shat\kernmat(\shat\kernmat^{-1} - \invapprox_{\idxiter-1})\shat{\kernmat}\qoi = (\shat{\kernmat} - \kernmatapprox_{\idxiter-1})\qoi.
	\end{align*}
\end{proof}

\begin{proposition}[Batch of Observations]
	\label{prop:problinsolve-batch-posterior}
	Let \(\qoicov_0\) such that \(\qoicov_0 \shat{\kernmat} \action_j
	= \action_j\) for all \(j \in \{1, \dots, \idxiter\}\). Then after \(\idxiter\) iterations the
	posterior over the representer weights in \eqref{eqn:posterior-representer-weights} is equivalent to the one computed
	for a batch of observations, i.e.
	\begin{align*}
		\qoimean_\idxiter & = \qoicov_0 \shat{\kernmat} \mActions_\idxiter (\mActions_\idxiter^\top \shat{\kernmat} \qoicov_0 \shat{\kernmat}\mActions_\idxiter)^{-1} \mActions_\idxiter^\top(\labels - \vmu)                      \\
		\qoicov_\idxiter  & = \qoicov_0 - \qoicov_0 \shat{\kernmat} \mActions_\idxiter (\mActions_\idxiter^\top \shat{\kernmat} \qoicov_0 \shat{\kernmat}\mActions_\idxiter)^{-1} \mActions_\idxiter^\top \shat{\kernmat}\qoicov_0
	\end{align*}
\end{proposition}

\begin{proof}
	This can be seen as a direct consequence of recursively applying Bayes' theorem
	\begin{equation*}
		p(\qoi \mid \{\observ_\idxiter\}_{\idxiter=1}^m, \{\action_\idxiter\}_{\idxiter=1}^m) =
		\frac{p(\observ_m \mid \action_m, \qoi)p(\qoi \mid \{\observ_\idxiter\}_{\idxiter=1}^{m-1},
		\{\action_\idxiter\}_{\idxiter=1}^{m-1})}{\int p(\observ_m \mid \action_m, \qoi)p(\qoi \mid
		\{\observ_\idxiter\}_{\idxiter=1}^{m-1}, \{\action_\idxiter\}_{\idxiter=1}^{m-1}) d \qoi}.
	\end{equation*}

	However, here we also give a geometric proof based on the projection property of the precision
	matrix approximation \(\invapprox_\idxiter\). By using \cref{eqn:invapprox-batch-form} and the assumption
	on \(\qoicov_0\) we have that
	\begin{align*}
		\invapprox_\idxiter = \mActions_\idxiter (\mActions_\idxiter^\top  \shat{\kernmat}\mActions_\idxiter)^{-1} \mActions_\idxiter^\top & = \qoicov_0 \shat{\kernmat} \mActions_\idxiter (\mActions_\idxiter^\top \shat{\kernmat} \qoicov_0 \shat{\kernmat}\mActions_\idxiter)^{-1} \mActions_\idxiter^\top                          \\
		                                                                                                                                   & = \qoicov_0 \shat{\kernmat} \mActions_\idxiter (\mActions_\idxiter^\top \shat{\kernmat} \qoicov_0 \shat{\kernmat}\mActions_\idxiter)^{-1} \mActions_\idxiter^\top \shat{\kernmat}\qoicov_0
	\end{align*}
	This proves that
	\begin{equation*}
		\qoicov_\idxiter = \qoicov_0 - \invapprox_\idxiter = \qoicov_0 - \qoicov_0 \shat{\kernmat}
		\mActions_\idxiter (\mActions_\idxiter^\top \shat{\kernmat} \qoicov_0 \shat{\kernmat}\mActions_\idxiter)^{-1} \mActions_\idxiter^\top \shat{\kernmat}\qoicov_0
	\end{equation*}
	Now by \cref{eqn:relation-precision-matrix-qoimean} it holds that \(\invapprox_\idxiter(\labels - \vmu) =
	\qoimean_\idxiter\). This proves the claim.
\end{proof}

\begin{restatable}[Posterior Contraction]{proposition}{posteriorcontraction}
	\label{prop:posterior-contraction}
	Let \(\mActions_\idxiter \in \R^{n \times \idxiter}\) be the actions
	chosen by
	\Cref{alg:itergp}, then its
	posterior contracts as
	\begin{equation*}
		\tr(\qoicov_\idxiter \qoicov_0^{-1}) = \tr(\qoicov_\idxiter \shat{\kernmat}) =  n - \rank(\mActions_\idxiter).
	\end{equation*}
\end{restatable}
\begin{proof}
	Since \(\mSigma_0 = \hat{\mK}^{-1}\), we have by \cref{eqn:invapprox-batch-form}, that
	\begin{align*}
		\tr(\mSigma_\idxiter \mSigma_0^{-1}) & = \tr((\mSigma_0 - \invapprox_\idxiter) \shat{\kernmat})                                                                                                                                                                   \\
		                                     & = \tr(\mI_n - \mActions_\idxiter(\mActions_\idxiter^\top \shat{\kernmat} \mActions_\idxiter)^{\dagger}\mActions_\idxiter^\top \shat{\kernmat})                                                                             \\
		                                     & = \tr(\mI_n) - \operatorname{tr}(\underbracket[0.14ex]{\mActions_\idxiter^\top \shat{\kernmat}\mActions_\idxiter(\mActions_\idxiter^\top \shat{\kernmat} \mActions_\idxiter)^\dagger}_{\in \R^{\idxiter \times \idxiter}}) \\
		                                     & = n - \rank(\mActions_\idxiter)
	\end{align*}
	Now, if the actions \(\mS_\idxiter\) are chosen linearly independent, then
	\(\rank(\mS_\idxiter) = \idxiter\).
\end{proof}

\begin{theorem}[Online GP Approximation with Algorithm 1]
	\label{thm:online-inference}
	Let \(n, n' \in \N\) and consider training data sets \(\mX \in \R^{n \times d}, \vy \in
	\R^n\) and \(\mX' \in \R^{n' \times d}, \vy' \in \R^{n'}\). Consider two
	sequences of actions \((\action_\idxiter)_{\idxiter=1}^n \in
	\R^{n}\) and \((\tilde\action_\idxiter)_{\idxiter=1}^{n+n'} \in
	\R^{n+n'}\) such that for all \(i \in \{1, \dots, n\}\), it holds that
	\begin{equation}
		\label{eqn:proof-online-inference-actions}
		\tilde\action_\idxiter = \begin{pmatrix} \action_\idxiter\\  \vzero\end{pmatrix}
	\end{equation}
	Then the posterior returned by \Cref{alg:itergp} for the dataset \((\mX, \vy)\) using actions
	\(\action_\idxiter\) is identical to the posterior returned by \Cref{alg:itergp} for the
	extended dataset using actions \(\tilde\action_\idxiter\), i.e. it holds for any \(i \in \{1, \dots,
	n\}\), that
	\begin{equation*}
		\textsc{IterGP}(\mu, k, \mX, \vy, (\action_\idxiter)_\idxiter) = (\mu_i, k_i) = (\tilde\mu_i, \tilde{k}_i) =
		\textsc{IterGP}\left(\mu, k, \begin{pmatrix} \mX\\ \mX'\end{pmatrix}, \begin{pmatrix} \vy \\ \vy' \end{pmatrix},
		(\tilde\action_\idxiter)_\idxiter \right).
	\end{equation*}
\end{theorem}

\begin{proof}
	Define \(\tilde{\mX} = \begin{pmatrix} \mX\\ \mX'\end{pmatrix}\) and \(\tilde{\vy} =
	\begin{pmatrix} \vy \\ \vy' \end{pmatrix}\). We begin by showing that the search directions of both methods
	satisfy
	\begin{equation}
		\label{eqn:proof-online-inference-searchdirs}
		\searchdir_i' = \begin{pmatrix} \searchdir_\idxiter\\  \vzero\end{pmatrix}.
	\end{equation}
	We proceed by induction. For \(\idxiter = 0\) it holds by definition of \Cref{alg:itergp} and
	\cref{eqn:proof-online-inference-actions} that
	\begin{equation}
		\tilde\searchdir_0 = \tilde\action_0 = \begin{pmatrix} \action_0\\  \vzero\end{pmatrix} =  \begin{pmatrix} \searchdir_0\\  \vzero\end{pmatrix}.
	\end{equation}
	Now for the induction step \(\idxiter \to \idxiter + 1\), assume that \eqref{eqn:proof-online-inference-searchdirs}
	holds for \(j \in \{1, \dots, \idxiter\}\). Then, we have
	\begin{align*}
		\tilde\searchdir_{\idxiter + 1} & = \tilde\qoicov_{\idxiter-1} (k(\tilde{\mX}, \tilde{\mX}) + \sigma^2 \mI_{n+n'}) \tilde\action_{\idxiter + 1} \\  
		&= (\mI_{n+n'} - \tilde\invapprox_{\idxiter} (k(\tilde{\mX}, \tilde{\mX}) + \sigma^2 \mI_{n+n'})) \tilde\action_{\idxiter + 1}\\
		&= \tilde\action_{\idxiter + 1} - \sum_{j=1}^{\idxiter}\frac{1}{\tilde\searchdirsqnorm_j}\tilde\searchdir_j (\tilde\searchdir_j)^\top (k(\tilde{\mX}, \tilde{\mX}) + \sigma^2 \mI_{n+n'}) \tilde\action_{\idxiter + 1}\\
		&\overset{\textup{IH}}{=} \begin{pmatrix} \action_{\idxiter+1}\\  \vzero\end{pmatrix} - \sum_{j=1}^{\idxiter}\frac{1}{\tilde\searchdirsqnorm_j}\begin{pmatrix} \searchdir_j\\  \vzero\end{pmatrix}\begin{pmatrix} \searchdir_j^\top &  \vzero\end{pmatrix} \begin{pmatrix} k(\mX, \mX) + \mI_{n} & k(\mX, \mX')\\ k(\mX', \mX) & k(\mX', \mX') + \mI_{n'} \end{pmatrix} \begin{pmatrix} \action_{\idxiter+1}\\  \vzero\end{pmatrix}\\
		&= \begin{pmatrix} \action_{\idxiter+1} - \sum_{j=1}^{\idxiter}\frac{1}{\searchdirsqnorm_j}\searchdir_j (\searchdir_j)^\top \shat\kernmat \action_{\idxiter+1}\\ \vzero\end{pmatrix}\\
		&= \begin{pmatrix} \searchdir_{\idxiter+1}\\  \vzero\end{pmatrix}
	\end{align*}
	where we used that \(\tilde\searchdirsqnorm_j = \tilde\action_j^\top (k(\tilde{\mX}, \tilde{\mX}) + \sigma^2 \mI_{n+n'}) \tilde\searchdir_j = \action_j^\top \shat\kernmat \searchdir_j = \searchdirsqnorm_j\). This proves \cref{eqn:proof-online-inference-searchdirs}. Now recognize that 
	\begin{align*}
		\tilde\observ_j = \tilde\action_j^\top \tilde\residual_j &= \tilde\action_j^\top (\tilde\vy - \tilde\vmu - \tilde\kernmat \tilde\invapprox_\idxiter (\tilde\vy - \tilde\vmu))\\
		&=\tilde\action_j^\top (\tilde\vy - \tilde\vmu - (\tilde\kernmat +\sigma^2\mI)\sum_{\ell=1}^{j}\frac{1}{\tilde\searchdirsqnorm_\ell}\tilde\searchdir_\ell \tilde\searchdir_\ell^\top (\tilde\vy - \tilde\vmu))\\
		&= \action_j^\top(\vy - \vmu) - \sum_{\ell=1}^{j}\frac{1}{\searchdirsqnorm_\ell} \action_j^\top \shat{\kernmat}\searchdir_\ell \searchdir_\ell^\top (\vy - \vmu)\\
		&= \action_j^\top(\labels - \vmu - \shat{\kernmat} \invapprox_j(\labels - \vmu))\\
		&= \action_j^\top \residual_j\\
		&= \observ_j
	\end{align*}
	Therefore, we finally have that
	\begin{align*}
		\tilde\mu_i(\cdot) &= \mu(\cdot) + k(\cdot, \tilde\mX) \tilde\qoimean_\idxiter = \mu(\cdot) + k(\cdot, \tilde\mX) \sum_{j=1}^\idxiter \frac{\tilde\observ_j}{\tilde\searchdirsqnorm_j} \tilde\searchdir_j\\
		&= \mu(\cdot) + k(\cdot, \mX) \qoimean_\idxiter\\
		\intertext{as well as}
		\tilde{k}_i(\cdot, \cdot) &= k(\cdot, \cdot) - k(\cdot, \tilde\mX) \tilde\invapprox_\idxiter k(\tilde \mX, \cdot) =  k(\cdot, \cdot) - k(\cdot, \tilde\mX) \sum_{j=1}^{\idxiter}\frac{1}{\tilde\searchdirsqnorm_j}\tilde\searchdir_j (\tilde\searchdir_j)^\top k(\tilde \mX, \cdot)\\
		&= k(\cdot, \cdot) - k(\cdot, \traindata) \sum_{j=1}^{\idxiter}\frac{1}{\searchdirsqnorm_j}\searchdir_j (\searchdir_j)^\top k( \mX, \cdot)= k(\cdot, \cdot) - k(\cdot, \traindata) \invapprox_\idxiter k( \mX, \cdot) = k_i(\cdot, \cdot).
	\end{align*}
\end{proof}

\begin{remark}[Streaming Gaussian Processes]
	\Cref{thm:online-inference} shows that any variant of IterGP can be used in
	the online setting where data arrives sequentially \emph{while} the algorithm is
	running. Now, if we assume data points arrive one at a time, we choose unit vector actions
	(IterGP-Chol) and perform one iteration of \Cref{alg:itergp} after each data point, then
	\Cref{alg:itergp} simply computes the mathematical GP posterior.
\end{remark}

\subsection{Approximation of Representer Weights}
\label{suppsec:approx-representer-weights}

\representerweightserror*

\begin{proof}
	Define \(\mH_\idxiter = \qoicov_{\idxiter}\shat{\kernmat} = \mI - \invapprox_\idxiter
	\shat{\kernmat}\). We have by \Cref{lem:relation-precision-matrix-qoimean}, that
	\begin{align*}
		\norm{\qoi - \qoimean_\idxiter}_{\shat{\kernmat}}^2 & = \norm{\mH_\idxiter\qoi}_{\shat{\kernmat}}^2 = (\mH_\idxiter \qoi)^\top \shat{\kernmat} \mH_\idxiter \qoi \overset{\eqref{eqn:hi-k-hi}}{=} \qoi^\top \mH_\idxiter \qoi = \bar{\vv}_*^\top \mH_\idxiter \bar{\vv}_* \norm{\qoi}_{\shat{\kernmat}}^2
	\end{align*}
	This proves the first equality of \Cref{prop:representer-weights-error}. Further it holds that
	\begin{align*}
		\norm{\mH_\idxiter\qoi}_{\shat{\kernmat}} & = \lVert{\shat{\kernmat}^{\frac{1}{2}}\mH_\idxiter\qoi}\rVert_2 = \lVert{(\mI - \shat{\kernmat}^{\frac{1}{2}} \invapprox_\idxiter\shat{\kernmat}^{\frac{1}{2}}) \shat{\kernmat}^{\frac{1}{2}}\qoi}\rVert_2 \leq \lVert{\mI - \shat{\kernmat}^{\frac{1}{2}} \invapprox_\idxiter\shat{\kernmat}^{\frac{1}{2}}}\rVert_2 \norm{\qoi}_{\shat{\kernmat}} \\
		                                          & = \lambda_{\max}(\mI - \shat{\kernmat}^{\frac{1}{2}} \invapprox_\idxiter\shat{\kernmat}^{\frac{1}{2}}) \norm{\qoi}_{\shat{\kernmat}}.
	\end{align*}
	Now by Weyl's inequality and the fact that
	\(\shat{\kernmat}^{\frac{1}{2}}\invapprox_{\idxiter}
	\shat{\kernmat}^{\frac{1}{2}}\) is positive semi-definite, it holds that
	\begin{align*}
		\lambda_{\max}(\mH_\idxiter) & = \lambda_{\max}(\mI - \shat{\kernmat}^{\frac{1}{2}} \invapprox_\idxiter\shat{\kernmat}^{\frac{1}{2}}) \leq \lambda_{\max}(\mI) - \lambda_{\min}(\shat{\kernmat}^{\frac{1}{2}} \invapprox_\idxiter\shat{\kernmat}^{\frac{1}{2}})\leq 1.
	\end{align*}
	Now, recall that similar matrices \(\mA\) and \(\mB = \mP^{-1} \mA \mP\) have the same eigenvalues. Therefore 
	\begin{equation*}
		\mI - \shat{\kernmat}^{\frac{1}{2}} \invapprox_\idxiter\shat{\kernmat}^{\frac{1}{2}} = \shat{\kernmat}^{\frac{1}{2}}(\mI - \invapprox_\idxiter \shat\kernmat)\shat{\kernmat}^{-\frac{1}{2}}
	\end{equation*}
	and \(\mI - \invapprox_\idxiter \shat\kernmat\) have the same eigenvalues.
	Finally, since by \cref{eqn:sigmai-projection} \(\mH_\idxiter\) is a projection onto
	\(\linspan{\mActions_\idxiter}^{\perp_{\shat{\kernmat}}}\), it has full rank at iteration \(n\) if the
	actions are linearly independent and therefore \(\lambda_{\max}(\mH_n) = 1\). This proves the
	claim.
\end{proof}

\subsection{Convergence Analysis of the Posterior Mean Approximation}
\label{suppsec:convergence-posterior-mean}

\thmconvergencerkhs*

\begin{proof}

	Let \(\rho(\idxiter)\) such that \(\norm{\qoi - \qoimean_\idxiter}_{\shat{\kernmat}} \leq
	\rho(\idxiter)
	\norm{\qoi - \qoimean_0}_{\shat{\kernmat}}\), where \(\qoimean_0 = \vzero\). Then, we have for \(i \in \{0, \dots, n\}\),
	that
	\begin{align*}
		\norm{\qoi - \qoimean_\idxiter}_{\kernmat}^2 & \leq \norm{\qoi - \qoimean_\idxiter}_{\shat{\kernmat}}^2 \leq \rho(\idxiter)^2 \norm{\qoi - \qoimean_0}_{\shat{\kernmat}}^2 \\ &= \rho(\idxiter)^2 \big(\norm{\qoi - \qoimean_0}_{\kernmat}^2 + \sigma^2 \frac{1}{\lambda_{\min}(\kernmat)} \underbracket[0.14ex]{\lambda_{\min}(\kernmat)\norm{\qoi - \qoimean_0}_2^2}_{\leq \norm{\qoi - \qoimean_0}^2_\kernmat} \big)\\
		                                             & \leq \rho(\idxiter)^2 \left(1 + \frac{\sigma^2}{\lambda_{\min}(\kernmat)}\right)\norm{\qoi - \qoimean_0}_{\kernmat}^2
	\end{align*}
	Now by assumption \(\mu_\idxiter(\cdot) = \mu(\cdot) +
	\sum_{j=1}^n  (\vv_\idxiter)_j k(\cdot, \vx_j) = \mu(\cdot) + k(\cdot,
	\traindata) \invapprox_\idxiter \labels\). By the reproducing property we obtain for \(\Delta =
	\qoi - \qoimean_\idxiter\) that
	\begin{align*}
		\norm{\qoi - \qoimean_\idxiter}_{\kernmat}^2 & = \Delta^\top \kernmat \Delta                                                                                                                                                                         \\
		                                             & = \sum_{\ell=1}^n \sum_{j=1}^n \Delta_\ell \Delta_j k(\vx_\ell, \vx_j)                                                                                                                                \\
		                                             & =\sum_{\ell=1}^n \sum_{j=1}^n \Delta_\ell \Delta_j \langle k(\cdot, \vx_\ell),  k(\cdot, \vx_j) \rangle_{\mathcal{H}_k}                 & \text{\(k\) is the reproducing kernel of \(\mathcal{H}_k\)} \\
		                                             & =\langle \sum_{\ell=1}^n \Delta_\ell k(\cdot, \vx_\ell), \sum_{j=1}^n \Delta_j k(\cdot, \vx_j) \rangle_{\mathcal{H}_k}                                                                                \\
		                                             & = \norm{\sum_{\ell=1}^n \Delta_\ell k(\cdot, \vx_\ell)}_{\mathcal{H}_k}^2                                                                                                                             \\
		                                             & = \norm{\sum_{\ell=1}^n (\qoi)_\ell k(\cdot, \vx_\ell) - \sum_{\ell=1}^n (\qoimean_\idxiter)_\ell k(\cdot, \vx_\ell)}_{\mathcal{H}_k}^2                                                               \\
		                                             & = \norm{\mu_* - \mu_\idxiter}^2_{\mathcal{H}_k}                                                                                         & \text{See Theorem 3.4 in \citet{Kanagawa2018}}
	\end{align*}

	Combining the above and setting \(c(\sigma^2) = 1 + \frac{\sigma^2}{\lambda_{\min}(\kernmat)}\) we obtain
	\begin{equation*}
		\norm{\mu_* - \mu_\idxiter}_{\mathcal{H}_k} = \norm{\qoi - \qoimean_\idxiter}_{\kernmat} \leq
		\rho(\idxiter) c(\sigma^2) \norm{\qoi - \qoimean_0}_{\kernmat} = \rho(\idxiter) c(\sigma^2)
		\norm{\mu_* - \mu_0}_{\mathcal{H}_k}.
	\end{equation*}
\end{proof}

\subsection{Combined Uncertainty as Worst Case Error}
\label{suppsec:combined-uncertainty-worst-case}

\thmworstcaseerror*

\begin{proof}
	Let \(\vx_0 = \vx\), \(c_0=1\) and \(c_j = - (\invapprox_\idxiter
	k^\sigma(\traindata, \vx))_{j}\) for \(j=1, \dots n\), where
	\(k^\sigma(\cdot, \cdot) \coloneqq k(\cdot, \cdot) + \sigma^2 \delta(\cdot, \cdot)\). Then by
	Lemma 3.9 of \citet{Kanagawa2018}, it holds that
	\begin{align*}
		\bigg(\sup_{g \in \mathcal{H}_{k^\sigma} : \norm{g}_{\mathcal{H}_{k^\sigma}} \leq 1} & (g(\vx) - \mu_\idxiter^g(\vx)) \bigg)^2 = \bigg(\sup_{g \in \mathcal{H}_{k^\sigma} : \norm{g}_{\mathcal{H}_{k^\sigma}} \leq 1} \sum_{j=0}^n c_j g(\vx_j)\bigg)^2                                                                      \\
		                                                                                     & = \norm{k^\sigma(\cdot, \vx_0) - \sum_{j=1}^n k(\vx, \vx_j) \invapprox_\idxiter k^\sigma(\cdot, \vx_j)}_{\mathcal{H}_{k^\sigma}}^2                                                                                                    \\
		                                                                                     & = \norm{k^\sigma(\cdot, \vx) - k(\vx, \traindata) \invapprox_\idxiter k^\sigma(\traindata, \cdot)}_{\mathcal{H}_{k^\sigma}}^2                                                                                                         \\
		                                                                                     & = \langle k^\sigma(\cdot, \vx), k^\sigma(\cdot, \vx) \rangle_{\mathcal{H}_{k^\sigma}} - 2 \langle  k^\sigma(\cdot, \vx), k(\vx, \traindata) \invapprox_\idxiter k^\sigma(\traindata, \cdot)\rangle_{\mathcal{H}_{k^\sigma}}           \\ &\qquad+ \langle k(\vx, \traindata) \invapprox_\idxiter k^\sigma(\traindata, \cdot), k(\vx, \traindata) \invapprox_\idxiter k^\sigma(\traindata, \cdot) \rangle_{\mathcal{H}_{k^\sigma}}\\
		\intertext{Now by the reproducing property, it follows that}
		                                                                                     & = k^\sigma(\vx, \vx) - 2 k^\sigma(\vx, \traindata) \invapprox_{\idxiter} k^\sigma(\traindata, \vx) + k^\sigma(\vx, \traindata) \invapprox_{\idxiter} k^\sigma(\traindata, \traindata) \invapprox_{\idxiter} k^\sigma(\traindata, \vx)
		\intertext{If \(\sigma^2 > 0\) and \(\vx \neq \vx_j\) or if \(\sigma^2=0\), it holds that \(k^\sigma(\vx, \traindata) = k(\vx, \traindata)\). Further by definition \(k^\sigma(\traindata, \traindata) = \shat{\kernmat}\) and
			finally by \eqref{eqn:ci-k-ci}, it holds that \(\invapprox_\idxiter \shat{\kernmat}\invapprox_\idxiter =
			\invapprox_\idxiter\). Therefore we have}
		                                                                                     & = k(\vx, \vx) + \sigma^2 - 2 k(\vx, \traindata) \invapprox_{\idxiter} k(\traindata, \vx) + k(\vx, \traindata) \invapprox_{\idxiter} \shat{\kernmat} \invapprox_{\idxiter} k(\traindata, \vx)                                          \\
		                                                                                     & = k(\vx, \vx)  -  k(\vx, \traindata) \invapprox_{\idxiter} k(\traindata, \vx) + \sigma^2                                                                                                                                              \\
		                                                                                     & = k_\idxiter(\vx, \vx) + \sigma^2
	\end{align*}
	We prove \cref{eqn:computational-uncertainty-worst-case} by an analogous argument. Choose \(c_j \coloneqq
	((\shat{\kernmat}^{-1} - \invapprox_\idxiter)
	k^\sigma(\traindata, \vx))_{j}\). We have
	\begin{align*}
		 & \bigg(\sup_{g \in \mathcal{H}_{k^\sigma} : \norm{g}_{\mathcal{H}_{k^\sigma}} \leq 1} (\mu_*^g(\vx) - \mu_\idxiter^g(\vx)) \bigg)^2 = \bigg(\sup_{g \in \mathcal{H}_{k^\sigma} : \norm{g}_{\mathcal{H}_{k^\sigma}} \leq 1} \sum_{j=0}^n c_j g(\vx_j)\bigg)^2                                                                             \\
		 & = \norm{\sum_{j=1}^n k(\vx, \vx_j) (\shat{\kernmat}^{-1} - \invapprox_\idxiter)k^\sigma(\cdot, \vx_j)}_{\mathcal{H}_{k^\sigma}}^2                                                                                                                                                                                                       \\
		 & = \norm{k(\vx, \traindata) (\shat{\kernmat}^{-1} - \invapprox_\idxiter)k^\sigma(\traindata, \cdot)}_{\mathcal{H}_{k^\sigma}}^2                                                                                                                                                                                                          \\
		 & = k^\sigma(\vx, \traindata)\shat{\kernmat}^{-1}\shat{\kernmat} \shat{\kernmat}^{-1}k^\sigma(\traindata, \vx) - 2 k^\sigma(\vx, \traindata)\shat{\kernmat}^{-1}\shat{\kernmat} \invapprox_\idxiter k^\sigma(\traindata, \vx) + k^\sigma(\vx, \traindata)\invapprox_\idxiter\shat{\kernmat} \invapprox_\idxiter k^\sigma(\traindata, \vx) \\
		\intertext{Again, we use that \(k^\sigma(\vx, \traindata) = k(\vx, \traindata)\) by assumption and \eqref{eqn:ci-k-ci}. Therefore}
		 & = k(\vx, \traindata)(\shat{\kernmat}^{-1} - \invapprox_\idxiter)k(\traindata, \vx)                                                                                                                                                                                                                                                      \\
		 & = k_\idxiter^{\textup{comp}}(\vx, \vx)
	\end{align*}
	This concludes the proof.
\end{proof}

\section{Implementation of Algorithm 1}

\subsection{Policy Choice}
As illustrated in \Cref{fig:uncertainty-decomposition}, the choice of policy of \Cref{alg:itergp} determines where computation in input space is targeted and therefore where the combined posterior contracts first. However, the policy also determines whether the error in the posterior mean or (co-)variance are predominantly reduced first, as \Cref{fig:recovered-methods-illustration} shows (cf. IterGP-Chol and IterGP-PBR). Therefore the policy choice is application-dependent. If I am primarily interested in the predictive mean, I may select residual actions (IterGP-CG). If downstream I am making use of the predictive uncertainty, I may want to contract uncertainty globally as quickly as possible at the expense of predictive accuracy (IterGP-PI). Such a choice is not unique to IterGP, but necessary whenever we select a GP approximation. What IterGP adds is computation-aware, meaningful uncertainty quantification in the sense of \Cref{cor:worst-case-error} no matter the choice of policy.

\begin{figure}
	\includegraphics[width=\textwidth]{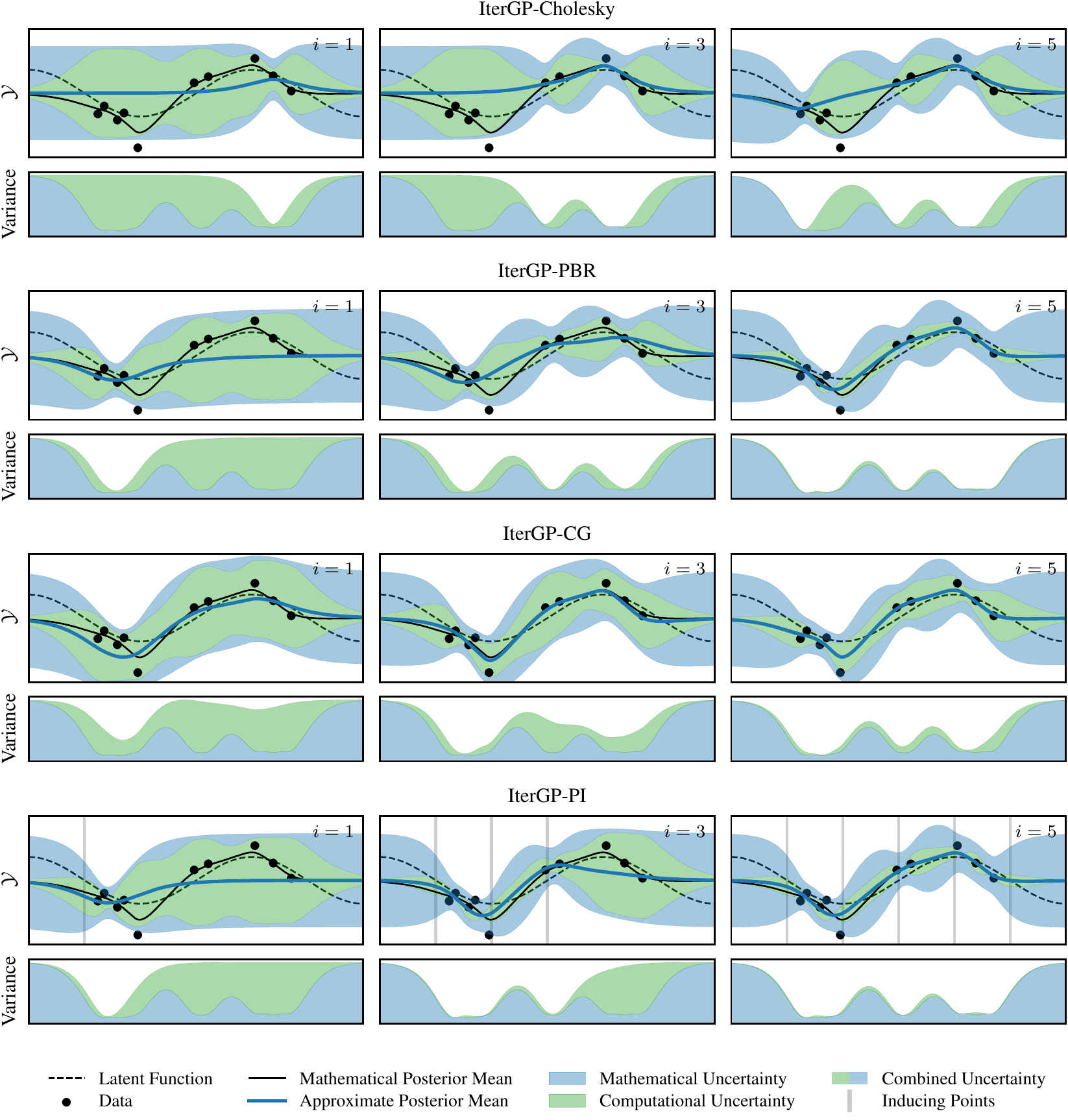}
	\caption{\emph{Illustration of IterGP analogs of commonly used GP approximations.}}
	\label{fig:recovered-methods-illustration}
\end{figure}

\subsection{Stopping Criterion}

In our implementation of \Cref{alg:itergp} we use the following two stopping criteria. Our computational budget can be directly controlled by specifying a \emph{maximum number of iterations}, since each iteration of IterGP needs the same number of matrix-vector multiplies. Alternatively, we terminate if the \emph{absolute or relative norm of the residual} are sufficiently small, i.e. if
\begin{equation}
	\norm{\residual_\idxiter}_2 < \delta_{\mathrm{abstol}} \qquad \text{or} \qquad \norm{\residual_\idxiter}_2 < \delta_{\mathrm{reltol}} \norm{\labels}_2.
\end{equation}
Of course other choices are possible. From a probabilistic numerics standpoint one may want to terminate once the combined marginal uncertainty at the training data is sufficiently small relative to the observation noise.

\subsection{Efficient Sampling from the Combined Posterior}
\label{suppsec:efficient-sampling}

Sampling from an exact GP posterior has cubic cost \(\bigO(n_\diamond^3)\) in the number of evaluation points \(n_\diamond\), which is
prohibitive for many useful downstream applications such as numerical integration over the
posterior using Monte-Carlo methods. \citet{Wilson2020,Wilson2020a} recently showed how to make use of
\emph{Matheron's rule} \cite{Journel1976,Chiles2009,Doucet2010} to efficiently sample from a GP posterior by
sampling from the prior and then performing a pathwise update. We can directly make use of this
strategy since \Cref{alg:itergp} computes a low-rank approximation to the precision matrix.
Assume we are given a draw \(f'_{\textup{prior}} \in
\mathcal{H}_k^\theta\) from the prior\footnote{In infinite dimensional reproducing kernel Hilbert spaces samples \(f \sim
	\GP(\mu, k)\) from a Gaussian process almost surely do not lie in the RKHS
	\(\mathcal{H}_k\) \citep[Cor.~4.10,][]{Kanagawa2018}. However, there exists \(f' \in \mathcal{H}_k^\theta\) in a larger RKHS
	\(\mathcal{H}_k^\theta \supset \mathcal{H}_k\) such that \(f'(\vx) =
	f(\vx)\) with probability 1 \citep[Thm.~4.12,][]{Kanagawa2018}.}
such that \(\labels'\sim \Normal(f'_{\textup{prior}}(\traindata), \sigma^2 \mI)\) constitutes a draw from the prior predictive. Then
\begin{equation}
	f'(\cdot) = f'_{\textup{prior}}(\cdot) + k(\cdot, \mX) \invapprox_\idxiter (\labels - \labels')
\end{equation}
is a draw from the combined posterior by Matheron's rule, which we can evaluate in
\(\bigO(n_\diamond n \idxiter)\) for \(n_\diamond\) evaluation points, since \(\invapprox_i\) has rank \(i\).

\section{Additional Experimental Results}
\label{sec:additional-experimental-results}

\begin{figure}[h!]
	\begin{subfigure}[b]{0.5\textwidth}
		\centering
		\includegraphics[width=\textwidth]{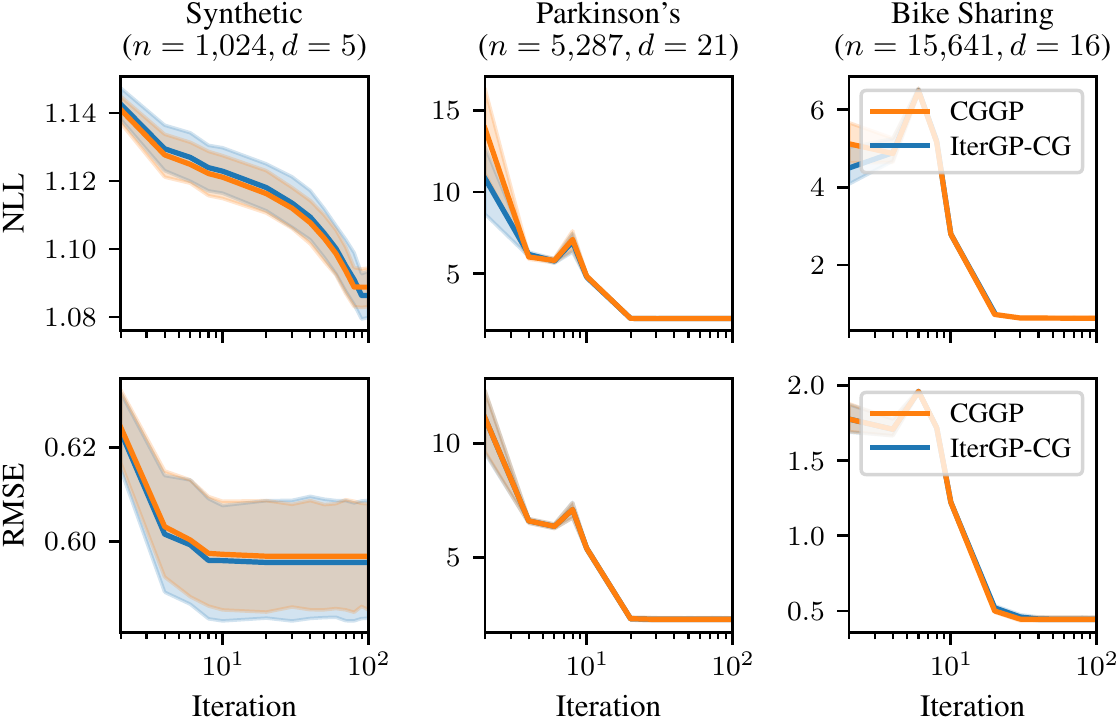}
		\caption{RBF kernel}
	\end{subfigure}\hfill
	\begin{subfigure}[b]{0.5\textwidth}
		\centering
		\includegraphics[width=\textwidth]{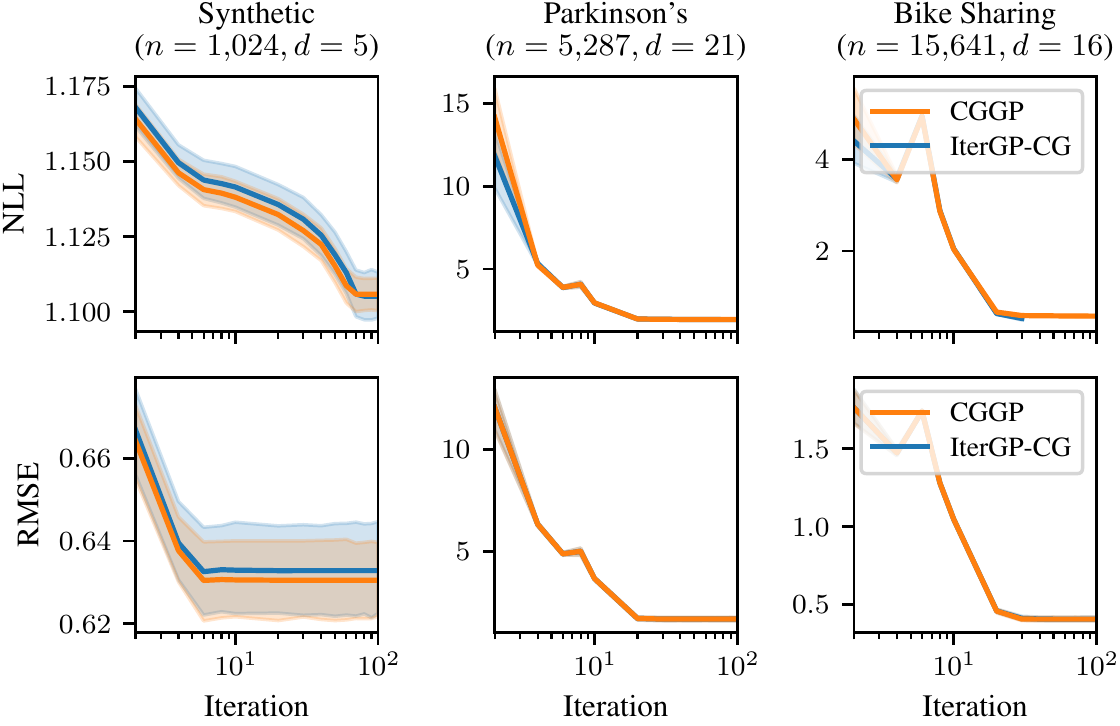}
		\caption{Mat\'ern(\(\frac{3}{2}\)) kernel}
	\end{subfigure}
	\caption{\emph{Generalization of CGGP and its closest IterGP analog.} GP regression using an RBF and Mat\'ern(\(\frac{3}{2}\)) kernel on UCI datasets. The plot shows the average generalization error in terms of NLL and RMSE for an
		increasing number of solver iterations. The posterior mean of IterGP-CG and CGGP is
		identical, which explains the identical RMSE.}
\end{figure}

\begin{figure}[h!]
	\begin{subfigure}[b]{\textwidth}
		\centering
		\includegraphics[width=\textwidth]{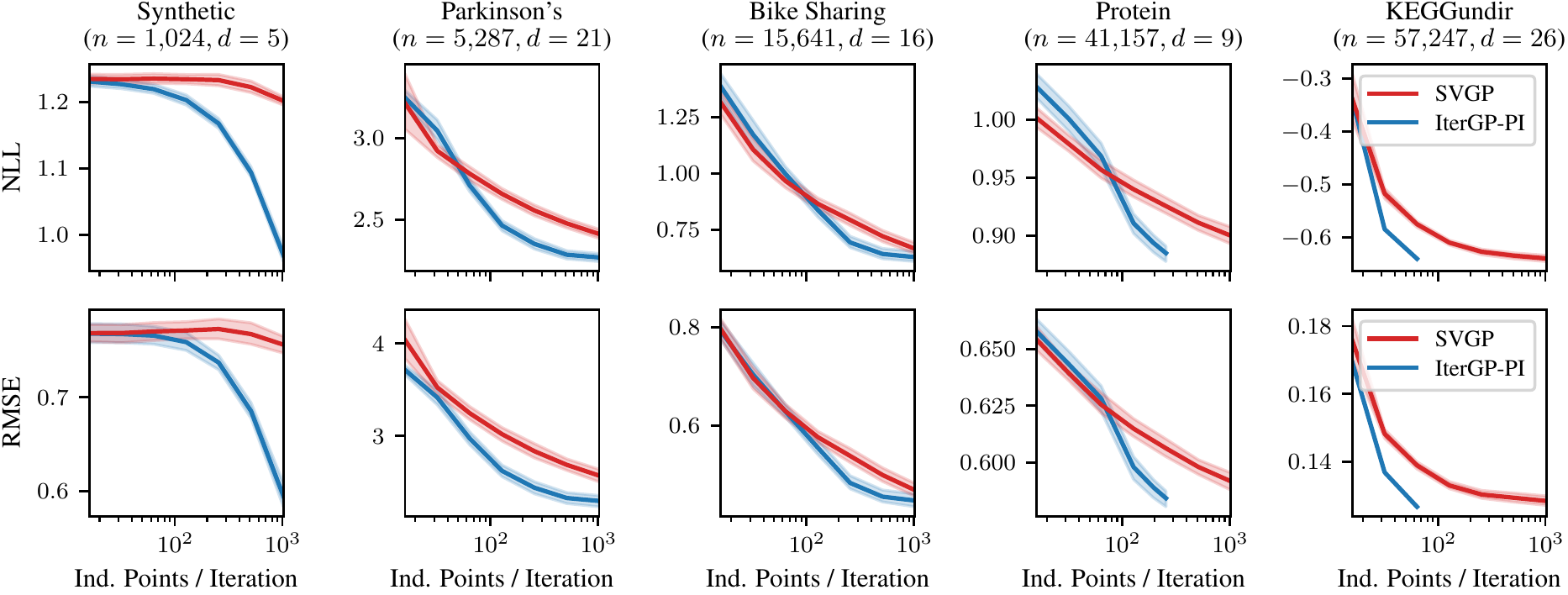}
		\caption{RBF kernel}
	\end{subfigure}\\[1em]
	\begin{subfigure}[b]{\textwidth}
		\centering
		\includegraphics[width=\textwidth]{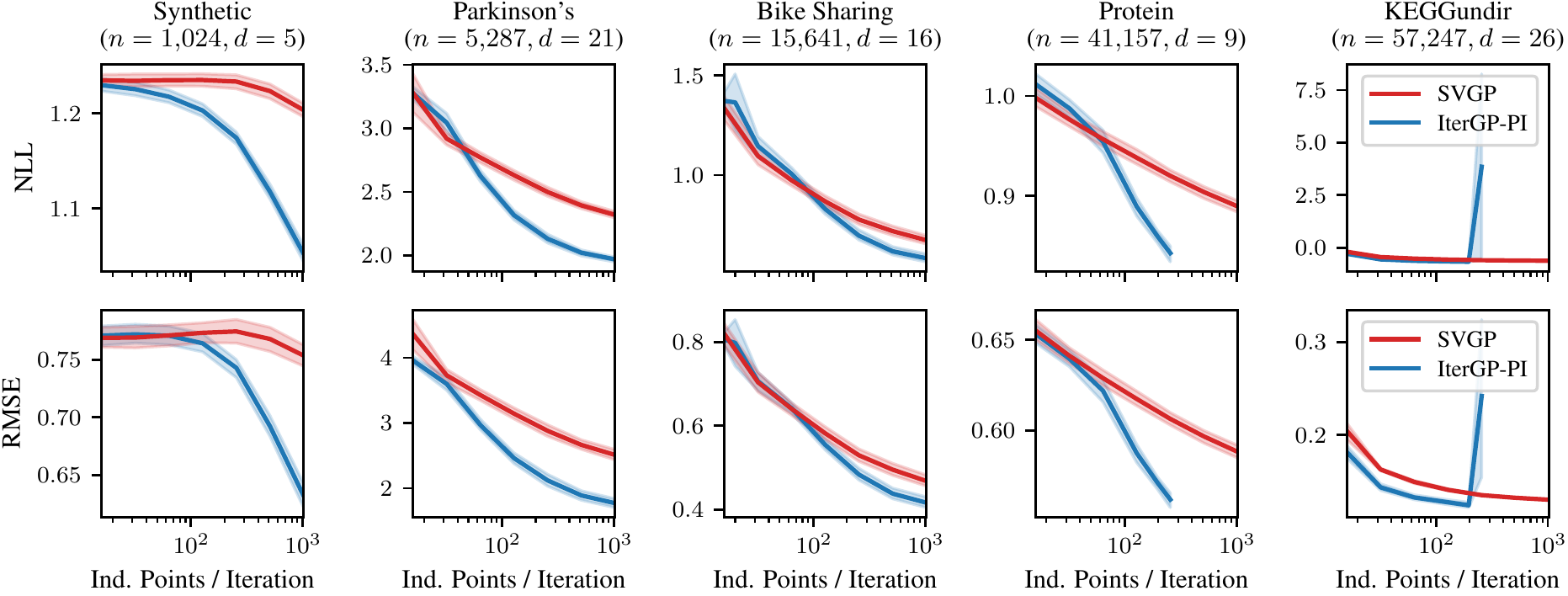}
		\caption{Mat\'ern(\(\frac{3}{2}\)) kernel}
	\end{subfigure}
	\caption{\emph{Generalization of SVGP and its closest IterGP analog.} GP regression using an RBF and Mat\'ern(\(\frac{3}{2}\)) kernel on UCI datasets. The plot shows the
		average generalization error in terms of NLL and RMSE for an increasing number of identical
		inducing points. After a small number of inducing points relative to the size of the training
		data, IterGP has significantly lower generalization error than SVGP. For the ``KEGGundir'' dataset after \(\approx 128\) iterations we observe numerical instability in some runs when computing the combined posterior of IterGP using a Mat\'ern(\(\frac{3}{2}\)) kernel.}
\end{figure}

\stopcontents[supplementary]

\end{document}